\documentclass{article}

\PassOptionsToPackage{numbers,sort&compress}{natbib}

\usepackage[preprint]{neurips_2026}

\usepackage[utf8]{inputenc} % allow utf-8 input
\usepackage[T1]{fontenc}    % use 8-bit T1 fonts
\usepackage{hyperref}       % hyperlinks
\usepackage{url}            % simple URL typesetting
\usepackage{booktabs}       % professional-quality tables
\usepackage{amsfonts}       % blackboard math symbols
\usepackage{nicefrac}       % compact symbols for 1/2, etc.
\usepackage{microtype}      % microtypography
\usepackage{xcolor}         % colors
 \usepackage{graphicx}
 \usepackage{caption}
 \usepackage{booktabs}
\usepackage{colortbl}
\usepackage{enumitem}
\usepackage{subcaption}
\usepackage{amsmath}
\usepackage{cleveref}
\usepackage{bbm}
\usepackage{amsthm}
\usepackage{wrapfig}
\usepackage[normalem]{ulem}
\usepackage{tikz}
\usepackage{graphicx}
\usepackage{pifont}
\usetikzlibrary{arrows.meta,positioning,calc,fit,shapes.geometric}

\usepackage{etoc}
\renewcommand{\contentsname}{Table of Contents}

\theoremstyle{plain}
\newtheorem{theorem}{Theorem}[section]
\newtheorem{lemma}[theorem]{Lemma}

\newtheorem{prop}[theorem]{Proposition}

\newtheorem*{lemmaformal}{Lemma~\ref{lem:main-coverage}}

\theoremstyle{definition}

\newtheorem{example}[theorem]{Example}

\theoremstyle{remark}

\definecolor{kone}{HTML}{FFE34D}
\definecolor{ktwo}{HTML}{69A7F2}
\definecolor{kthree}{HTML}{F06ACB}
\definecolor{kfour}{HTML}{20C2A7}
\definecolor{kfive}{HTML}{FF9B2F}
\definecolor{edgegray}{HTML}{4A4F57}
\definecolor{lightpanel}{HTML}{FAFAFA}
\definecolor{panelborder}{HTML}{B8B8B8}

\title{Compositional Reasoning in Language Models under Reinforcement Learning Post-Training}

\author{Yu He \\
  Stanford University\\
  \texttt{heyu@stanford.edu} \\
  \And
  Yingxi Li \\
  Stanford University\\
  \texttt{yingxi@stanford.edu} \\
  \And
  Yifei Wang \\
  Amazon AGI Labs\\
  \texttt{yifeiwg@amazon.com} \\
  \And
  Ellen Vitercik \\
  Stanford University\\
  \texttt{vitercik@stanford.edu} \\
}

\begin{document}

\maketitle

\begin{abstract}

Compositional reasoning is critical for real-world problem solving: since training data is necessarily limited, models must generalize by composing learned skills in new ways. While post-training methods such as reinforcement learning (RL) have substantially improved the reasoning abilities of language models (LMs), their effects on compositional reasoning remain less well understood. We propose a dependency-graph framework to formalize compositional reasoning, yielding three levels of compositionality with increasing complexity. Empirically, we instantiate this framework with data-structure tasks, which provide deterministic reward computation and clear compositional structure. We find a consistent decomposed-to-composed asymmetry: decomposed-skill training does not reliably transfer to composed tasks, whereas composed-task training transfers more readily back to decomposed tasks. We provide theoretical explanation for this asymmetry, and further evaluate compositional generalization under length extrapolation, structural distribution shift, and transfer to tasks requiring unseen skills. Finally, we present a pilot study on real-world tool-calling benchmarks, showing preliminary evidence that the decomposed-to-composed asymmetry can extend to practical settings.

\end{abstract}

\section{Introduction}

Compositional reasoning is a prerequisite for general intelligence. A model cannot merely learn individual skills; it must reuse and combine them systematically to solve unseen, more complex problems. This is necessary because real-world reasoning is open-ended, while training data is inherently limited. A model cannot be trained on every possible math problem, program, or tool-use trajectory it may encounter; instead, it must generalize by composing familiar skills in new ways.

As language models (LMs) take on increasingly complex reasoning tasks, understanding their compositional reasoning ability has become critical. However, existing studies of compositionality in LMs have largely focused on semantic composition in language understanding~\citep{yu2024skillmix, zhao2024can} or instruction following~\citep{yang-etal-2024-exploring}. In contrast, compositionality in reasoning tasks is less understood. Unlike linguistic composition, where primitives can cleanly correspond to words, phrases, or instructions, reasoning tasks often involve latent skills whose definitions and interactions are non-trivial to formalize. This makes it difficult to systematically evaluate whether models can compose reasoning skills.

On the other hand, reinforcement learning (RL) post-training has recently become a central approach for improving LM reasoning~\citep{deepseekai2025deepseekr1incentivizingreasoningcapability, shao2024deepseekmathpushinglimitsmathematical, lambert2025tulu, yue2026does}. Yet, most work on compositional reasoning studies in-context learning~\citep{ahuja2025provablelengthcompositionalgeneralization, ramesh2024compositional, chang2026characterizing, ye2025how} or supervised finetuning~\citep{zhao2024can, yang-etal-2024-exploring, yin2025learning}. Much less is known about the effects of RL on compositional reasoning. For example, one natural approach is to train on individual skills, where the smaller tasks may make it easier to design RL rewards and verifiers. However, it remains unclear whether LMs can compose the learned skills at test time.

To address this gap, we formalize compositional reasoning using a dependency-graph framework. We define skills as reusable primitives required to solve reasoning tasks, and use dependency graphs to model the compositional structure among them. This framework provides a formal basis for systematically evaluating compositional reasoning in language models.

For empirical evaluation, we instantiate the framework using data-structure tasks from DSR-Bench~\citep{he2026llmsreasonstructurallybenchmarking}, a benchmark for evaluating structural reasoning in LMs. Data structures are well suited for this study because they provide deterministic, verifiable final outputs for RL reward computation. Moreover, these algorithmic tasks can be decomposed into clear skills, and the way these skills compose is explicit and unambiguous. For example, in binary-search-tree construction, insertion is the primitive skill, and construction requires applying this skill sequentially. This controlled setting allows us to compare decomposed-skill training with composed-task training and analyze how RL post-training affects the compositional reasoning ability of LMs.

\textbf{A summary of our contributions:}
\begin{itemize}[leftmargin=*,topsep=2pt, itemsep=2pt, parsep=1pt]
    \item Our conceptual contribution is a framework formalizing compositional reasoning in LMs by representing tasks as dependency graphs over reusable primitive skills (Figure \ref{fig:three-compositionality}). This framework gives rise to three increasingly difficult levels of compositionality with distinct challenges: single-skill chains, which isolate horizon generalization; multi-skill chains, which add skill switching; and branch--merge graphs, which further add non-local dependencies across branches (Section~\ref{sec:compositionality-framework}).

    \item Empirically, we conduct a controlled RL comparison between decomposed-skill training and composed-task training across three compositionality levels, and identify a consistent asymmetry: decomposed-skill training transfers poorly to composed tasks, whereas composed-task training transfers more readily back to decomposed tasks. We further evaluate compositional generalization under three transfer settings: length generalization, distribution shift, and unseen skills (Section~\ref{sec:empirical-compositionality}). 

    \item We offer a theoretical explanation for the decomposed-to-composed asymmetry through error compounding and coverage shift, showing that composed reasoning requires sustained correctness over induced state distributions that decomposed supervision may not cover (Section~\ref{sec:theory}).

    \item Finally, we conduct a pilot study on real-world tool-calling benchmarks, demonstrating preliminary evidence for the decomposed-to-composed asymmetry in practical settings (Section~\ref{sec:pilot}).
\end{itemize}

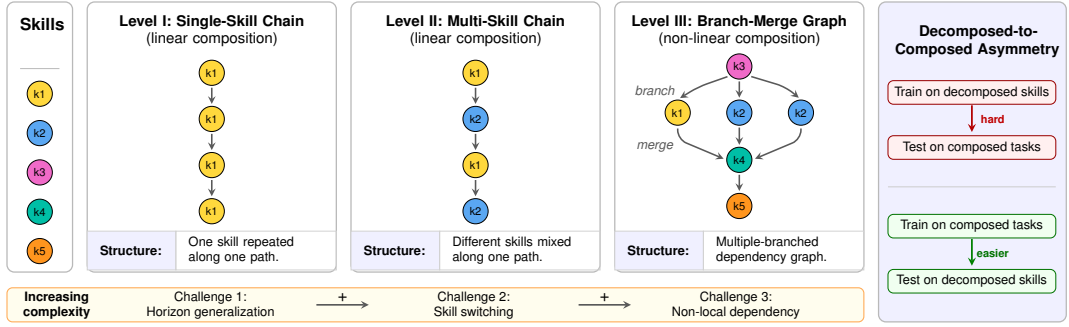
\begin{figure*}[t]
\centering
\resizebox{\linewidth}{!}{\begin{tikzpicture}[
    font=\sffamily,
    panel/.style={
        draw=gray!55,
        rounded corners=4pt,
        line width=0.7pt,
        fill=white
    },
    rowfill/.style={
        fill=blue!6,
        draw=gray!35,
        line width=0.45pt
    },
    rowbox/.style={
        draw=gray!35,
        line width=0.45pt
    },
    labeltext/.style={
        anchor=west,
        align=left,
        font=\sffamily\bfseries\scriptsize
    },
    bodytext/.style={
        anchor=west,
        align=left,
        font=\sffamily\scriptsize,
        text width=2.95cm
    },
    skillnode/.style={
        circle,
        draw=black,
        line width=0.55pt,
        minimum size=0.48cm,
        font=\sffamily\tiny,
        inner sep=0pt
    },
    arrow/.style={
        ->,
        >=stealth,
        line width=0.85pt,
        draw=black!65,
        shorten <=1.0pt,
        shorten >=1.0pt
    },
    brancharrow/.style={
        ->,
        >=stealth,
        line width=0.85pt,
        draw=black!65,
        rounded corners=5pt,
        shorten <=1.5pt,
        shorten >=1.5pt
    },
    title/.style={
        font=\sffamily\bfseries\footnotesize,
        align=center,
        text width=4.35cm
    },
    smallnote/.style={
        font=\sffamily\itshape\scriptsize,
        text=black!65
    },
    badge/.style={
        draw=gray!40,
        rounded corners=2pt,
        fill=gray!12,
        inner xsep=4pt,
        inner ysep=1.2pt,
        font=\sffamily\bfseries\tiny
    },
    obsbox/.style={
        rounded corners=3pt,
        line width=0.55pt,
        fill=white
    },
    hardbox/.style={
        draw=red!60!black,
        fill=red!6
    },
    easybox/.style={
        draw=green!45!black,
        fill=green!6
    },
    hardarrow/.style={
        ->,
        >=stealth,
        line width=0.95pt,
        draw=red!70!black
    },
    easyarrow/.style={
        ->,
        >=stealth,
        line width=0.95pt,
        draw=green!50!black
    },
    hierarrow/.style={
        ->,
        >=stealth,
        line width=0.9pt,
        draw=black!65
    }
]

\path[use as bounding box] (0,-0.15) rectangle (20.05,5.95);

% ------------------------------------------------------------
% Colors
% ------------------------------------------------------------
\definecolor{kone}{HTML}{FFD83D}
\definecolor{ktwo}{HTML}{5AA7F2}
\definecolor{kthree}{HTML}{F06AC8}
\definecolor{kfour}{HTML}{00BFA6}
\definecolor{kfive}{HTML}{FF951E}

\draw[panel] (0,0.8) rectangle (1.25,5.95);

\node[
    font=\sffamily\bfseries\small,
    align=center,
    anchor=north west,
    text width=1.00cm
] at (0.05,5.78) {Skills};

\draw[gray!55] (0.25,4.68) -- (1.00,4.68);

\node[skillnode, fill=kone]   at (0.625,4.20) {k1};
% \node[anchor=west, font=\sffamily\scriptsize] at (1.03,4.20) {k1};

\node[skillnode, fill=ktwo]   at (0.625,3.45) {k2};
% \node[anchor=west, font=\sffamily\scriptsize] at (1.03,3.45) {k2};

\node[skillnode, fill=kthree] at (0.625,2.70) {k3};
% \node[anchor=west, font=\sffamily\scriptsize] at (1.03,2.70) {k3};

\node[skillnode, fill=kfour]  at (0.625,1.95) {k4};
% \node[anchor=west, font=\sffamily\scriptsize] at (1.03,1.95) {k4};

\node[skillnode, fill=kfive]  at (0.625,1.20) {k5};
% \node[anchor=west, font=\sffamily\scriptsize] at (1.03,1.20) {k5};

% ------------------------------------------------------------
% Main panel boxes
% ------------------------------------------------------------
\draw[panel] (1.53,0.8) rectangle (6.28,5.95);
\draw[panel] (6.56,0.8) rectangle (11.31,5.95);
\draw[panel] (11.59,0.8) rectangle (16.34,5.95);
% Rightmost panel spans full height and uses the same fill as "Structure"
\draw[panel, fill=blue!6] (16.62,-0.15) rectangle (20.2,5.95);

% ============================================================
% Panel (a): Single-skill linear composition
% ============================================================
% \node[badge] at (3.905,5.73) {Level 1};
\node[title] at (3.905,5.4) {Level I: Single-Skill Chain\\{\mdseries (linear composition)}};

\node[skillnode, fill=kone] (a1) at (3.905,4.60) {k1};
\node[skillnode, fill=kone] (a2) at (3.905,3.72) {k1};
\node[skillnode, fill=kone] (a3) at (3.905,2.84) {k1};
\node[skillnode, fill=kone] (a4) at (3.905,1.96) {k1};

\draw[arrow] (a1) -- (a2);
\draw[arrow] (a2) -- (a3);
\draw[arrow] (a3) -- (a4);

\draw[rowfill] (1.53,1.58) rectangle (3.22,0.84);
\draw[rowbox]  (3.22,1.58) rectangle (6.28,0.84);
\node[labeltext] at (1.65,1.21) {Structure:};
\node[bodytext]  at (3.34,1.21) {One skill repeated\\along one path.};

% \draw[rowfill] (1.53,0.84) rectangle (3.22,0.08);
% \draw[rowbox]  (3.22,0.84) rectangle (6.28,0.08);
% \node[labeltext] at (1.65,0.46) {Added\\difficulty:};
% \node[bodytext]  at (3.34,0.46) {Depth only:\\repeated error accumulation.};

% ============================================================
% Panel (b): Multi-skill linear composition
% ============================================================
% \node[badge] at (8.935,5.73) {Level 2};
\node[title] at (8.935,5.4) {Level II: Multi-Skill Chain\\{\mdseries (linear composition)}};

\node[skillnode, fill=kone]  (b1) at (8.935,4.60) {k1};
\node[skillnode, fill=ktwo]  (b2) at (8.935,3.72) {k2};
\node[skillnode, fill=kone]  (b3) at (8.935,2.84) {k1};
\node[skillnode, fill=ktwo]  (b4) at (8.935,1.96) {k2};

\draw[arrow] (b1) -- (b2);
\draw[arrow] (b2) -- (b3);
\draw[arrow] (b3) -- (b4);

\draw[rowfill] (6.56,1.58) rectangle (8.25,0.84);
\draw[rowbox]  (8.25,1.58) rectangle (11.31,0.84);
\node[labeltext] at (6.68,1.21) {Structure:};
\node[bodytext]  at (8.37,1.21) {Different skills mixed\\along one path.};

% \draw[rowfill] (6.56,0.84) rectangle (8.25,0.08);
% \draw[rowbox]  (8.25,0.84) rectangle (11.31,0.08);
% \node[labeltext] at (6.68,0.46) {Added\\difficulty:};
% \node[bodytext]  at (8.37,0.46) {Type I + skill switching\\and interface mismatch.};

% ============================================================
% Panel (c): Branch-and-merge composition
% ============================================================
% \node[badge] at (13.965,5.73) {Level 3};
\node[title] at (13.965,5.4) {Level III: Branch-Merge Graph\\{\mdseries (non-linear composition)}};

\node[skillnode, fill=kthree] (c0) at (13.965,4.72) {k3};
\node[skillnode, fill=kone]   (c1) at (12.790,3.84) {k1};
\node[skillnode, fill=ktwo]   (c2) at (13.965,3.84) {k2};
\node[skillnode, fill=ktwo]   (c3) at (15.140,3.84) {k2};
\node[skillnode, fill=kfour]  (c4) at (13.965,2.94) {k4};
\node[skillnode, fill=kfive]  (c5) at (13.965,2.06) {k5};

\draw[brancharrow] (c0.south west) -- (13.18,4.33) -- (c1.north);
\draw[arrow]       (c0) -- (c2);
\draw[brancharrow] (c0.south east) -- (14.85,4.33) -- (c3.north);

\draw[brancharrow] (c1.south) -- (12.79,3.33) -- (c4.west);
\draw[arrow]       (c2) -- (c4);
\draw[brancharrow] (c3.south) -- (15.14,3.33) -- (c4.east);

\draw[arrow] (c4) -- (c5);

\node[smallnote] at (12.35,4.27) {branch};
\node[smallnote] at (12.35,3.18) {merge};

\draw[rowfill] (11.59,1.58) rectangle (13.28,0.84);
\draw[rowbox]  (13.28,1.58) rectangle (16.34,0.84);
\node[labeltext] at (11.71,1.21) {Structure:};
\node[bodytext]  at (13.40,1.21) {Multiple-branched\\dependency graph.};

% \draw[rowfill] (11.59,0.84) rectangle (13.28,0.08);
% \draw[rowbox]  (13.28,0.84) rectangle (16.34,0.08);
% \node[labeltext] at (11.71,0.46) {Added\\difficulty:};
% \node[bodytext]  at (13.40,0.46) {Type II + non-linear\\dependency coordination.};

% ------------------------------------------------------------
% Bottom strip: hierarchical layering
% ------------------------------------------------------------
\draw[
    fill=yellow!8,
    draw=orange!65,
    rounded corners=3pt,
    line width=0.6pt
] (0,-0.15) rectangle (16.34,0.50);

\node[
    font=\sffamily\bfseries\scriptsize,
    align=center
] at (0.95,0.175) {Increasing\\complexity};

\node[
    font=\sffamily\scriptsize,
    align=center
] at (3.88,0.175) {Challenge 1: \\Horizon generalization};

\node[
    font=\sffamily\scriptsize,
    align=center
] at (8.9,0.175) {Challenge 2:\\Skill switching};

\node[
    font=\sffamily\scriptsize,
    align=center
] at (13.9,0.175) {Challenge 3:\\Non-local dependency};

\draw[hierarrow] (5.9,0.175) -- (6.9,0.175);
\node[
    font=\sffamily\small,
    align=center,
    fill=yellow!8,
    inner sep=1pt
] at (6.4,0.33) {+};

\draw[hierarrow] (10.9,0.175) -- (11.9,0.175);
\node[
    font=\sffamily\small,
    align=center,
    fill=yellow!8,
    inner sep=1pt
] at (11.4,0.33) {+};

% ============================================================
% Panel (d): Decomposed-to-Composed Asymmetry
% ============================================================
\node[title] at (18.5,5.2) {Decomposed-to-\\Composed Asymmetry\\{}};

% --- hard direction (top) ---
\draw[obsbox, hardbox] (16.8,4.00) rectangle (20.0,4.45);
\node[
    font=\sffamily\scriptsize,
    align=center
] at (18.4,4.225) {Train on decomposed skills};

\draw[obsbox, hardbox] (16.8,2.95) rectangle (20.0,3.40);
\node[
    font=\sffamily\scriptsize,
    align=center
] at (18.4,3.175) {Test on composed tasks};

\draw[hardarrow] (18.4,3.98) -- (18.4,3.47);
\node[
    font=\sffamily\bfseries\tiny,
    text=red!75!black,
    fill=none,
    inner sep=1pt
] at (18.8,3.72) {hard};

% divider
\draw[gray!40] (16.8,2.425) -- (20.0,2.425);

% --- easier direction (bottom) ---
\draw[obsbox, easybox] (16.8,1.45) rectangle (20.,1.90);
\node[
    font=\sffamily\scriptsize,
    align=center
] at (18.4,1.675) {Train on composed tasks};

\draw[obsbox, easybox] (16.8,0.40) rectangle (20.,0.85);
\node[
    font=\sffamily\scriptsize,
    align=center
] at (18.4,0.625) {Test on decomposed skills};

\draw[easyarrow] (18.4,1.43) -- (18.4,0.92);
\node[
    font=\sffamily\bfseries\tiny,
    text=green!45!black,
    fill=none,
    inner sep=1pt
] at (18.8,1.17) {easier};
\end{tikzpicture}}
\caption{Three levels of compositionality with increasing complexity, shown as dependency structures among skills. Skills are defined as reusable primitives required to solve a reasoning task.} 
\label{fig:three-compositionality}
\end{figure*}

\section{Related work}

\paragraph{Compositional generalization and skill composition.}
Compositional generalization has been studied extensively in sequence-to-sequence learning, including SCAN~\citep{pmlr-v80-lake18a}, CFQ~\citep{keysers2020measuring}, and COGS~\citep{kim-linzen-2020-cogs}, which evaluate whether models can recombine familiar primitives in novel semantic structures. Recent work extends this question to language models: SKILL-MIX~\citep{yu2024skillmix} evaluates whether models can generate text combining multiple specified skills; \citet{zhao2024can} study whether smaller models can learn higher-order skill combinations from lower-order examples; \citet{yang-etal-2024-exploring} study compositional generalization in instruction following; and \citet{yin2025learning} study composable chain-of-thought supervision for transferring from atomic to compositional tasks. These works study important forms of linguistic, instructional, or trace-level composition. In contrast, we focus on reasoning tasks, where the primitives and their interactions are often less explicit than surface-level linguistics.

\paragraph{RL post-training for compositional reasoning.}
RL post-training with verifiable rewards has recently become a central approach for improving language-model reasoning~\citep{deepseekai2025deepseekr1incentivizingreasoningcapability, shao2024deepseekmathpushinglimitsmathematical, lambert2025tulu, yue2026does}, but its role in compositional transfer remains less understood. Closest to our work, \citet{yuan2025f} show that RL can teach nested composition of string-transformation functions, while \citet{xu2026composition} propose Composition-RL, which constructs harder verifiable prompts by composing existing prompts for RL training. \citet{li2025unveiling} study compositional generalization in vision-language reasoning under SFT and RL. In contrast, we focus on text-based reasoning and offer a controlled study of compositional structure from an algorithmic-reasoning perspective.

\paragraph{Structural and algorithmic reasoning.}
\citet{dziri2023faith} use computation graphs to analyze fixed algorithm executions and show that transformer performance can degrade with compositional complexity. DSR-Bench~\citep{he2026llmsreasonstructurallybenchmarking} evaluates structural reasoning through deterministic data-structure tasks, and recent graph-reasoning work uses RL to improve LMs on graph-theoretic tasks~\citep{guo2026g1}. We build on this line of work by using data structures as a controlled environment for RL post-training, while introducing a dependency-graph framework that characterizes how reusable skills compose and isolates distinct bottlenecks: horizon generalization, skill switching, and non-local dependencies.

\section{A dependency-graph framework for compositionality} \label{sec:compositionality-framework}

We present our conceptual contribution by formalizing compositional reasoning with a \emph{dependency graph} over reusable skills (Figure~\ref{fig:three-compositionality}), capturing compositional structures in reasoning tasks. Our framework yields three hierarchical levels of compositionality that isolate distinct bottlenecks: horizon generalization, skill switching, and non-local dependencies. This structure also enables our analysis of decomposed-skill versus composed-task training. More broadly, it provides a formal basis for future empirical and theoretical studies of compositional reasoning in language models.

\paragraph{Dependency graph.} For a task instance $\tau \in \mathcal{T}$, let the dependency graph $G_\tau=(V_\tau,E_\tau)$ be a finite directed acyclic graph (DAG), where each node represents a skill invocation and each edge indicates that the output of one invocation is required as input to another. Intuitively, $G_\tau$ specifies \emph{which} intermediate computations must be performed, \emph{which} skill is invoked at each step, and \emph{how} they depend on one another; therefore, it represents the compositional structure within a task. We use a Binary Search Tree (BST) construction task over input $\{1, 2, 3\}$ as a running example.

\begin{itemize}[leftmargin=*,topsep=2pt, itemsep=2pt, parsep=1pt]
    \item \textbf{Skill.} We assume a set of reusable skill primitives $\mathcal{K}$. Each skill type $k \in \mathcal{K}$ is associated with a function $f_k : \mathcal{X}_k \rightarrow \mathcal{Y}_k$, where $\mathcal{X}_k$ and $\mathcal{Y}_k$ denote the input and output spaces of that skill. For example, BST construction is composed of a sequence of \texttt{insert} operations, where \texttt{insert} acts as the reusable skill primitive.

    \item \textbf{Node.} A node $v \in V_\tau$ is labeled by a skill type $\kappa(v) \in \mathcal{K}$, indicating which skill is invoked at that node. Let $z_v$ denote the intermediate output produced by node $v$. For example, a concrete operation of \texttt{(insert, 1)} to the current BST is a skill invocation.

    \item \textbf{Edge.} An edge $(u,v) \in E_\tau$ indicates a dependency constraint that the intermediate output $z_u$ is required to form the input to node $v$. Let $\mathrm{pred}(v)=\{u \in V_\tau : (u,v)\in E_\tau\}$ denote the predecessors of $v$. The input $x_v$ and output $z_v$ to node $v$ are computed via
\[
    x_v = \psi_v\!\left(\ell_v, \left(z_u\right)_{u \in \mathrm{pred}(v)}\right)
    \in \mathcal{X}_{\kappa(v)},
    \qquad
    z_v = f_{\kappa(v)}(x_v).
\] 
Here, $\ell_v$ denotes local information specified by the task instance (e.g., what value to insert to the tree) and $\psi_v$ is an input-construction map. For example, after \texttt{(insert, 1)}, the next operation \texttt{(insert, 2)} takes as input the BST state produced by the previous operation.

    \item \textbf{Output.} For simplicity, we assume that $G_\tau$ has a unique sink node $v_{\mathrm{sink}} \in V_\tau$. The final output of the task is the intermediate output $z_{v_{\mathrm{sink}}}$ produced at this sink node. For example, the final output of BST construction is given by the output of the skill invocation \texttt{(insert, 3)}.
\end{itemize}

\paragraph{Three levels of compositionality.} As illustrated in Figure~\ref{fig:three-compositionality}, we formalize three levels of compositionality under the dependency-graph framework.
\begin{itemize}[leftmargin=*,topsep=2pt, itemsep=2pt, parsep=1pt]

    \item \textbf{Level I: Single-skill chain (linear composition).}
    Single-skill chain is the simplest setting, where the same skill is repeated across a sequence of steps. Formally, $G_\tau$ is a path, all nodes share the same skill type, and $|\{\kappa(v) : v \in  V_{\tau}\}|=1$. This setting tests whether a model trained on smaller instances can extrapolate to larger ones that require more repetitions of the same reasoning pattern.

    \item \textbf{Level II: Multi-skill chain (linear composition).}
    Multi-skill chain also has a path-structured dependency graph, but the skill label $\kappa(v)$ may vary across nodes with $|\{\kappa(v) : v \in  V_{\tau}\}|>1$.  This setting tests whether a model can maintain intermediate state while sequentially composing heterogeneous skills, including handling skill switching and interface mismatches across skills.

    \item \textbf{Level III: Branch-merge graph (non-linear composition).} Here, $G_\tau$ is a general directed acyclic graph rather than a path. \emph{Branching} occurs when one intermediate output feeds multiple downstream nodes, while \emph{merging} occurs when a node depends on multiple upstream outputs. 

\end{itemize}

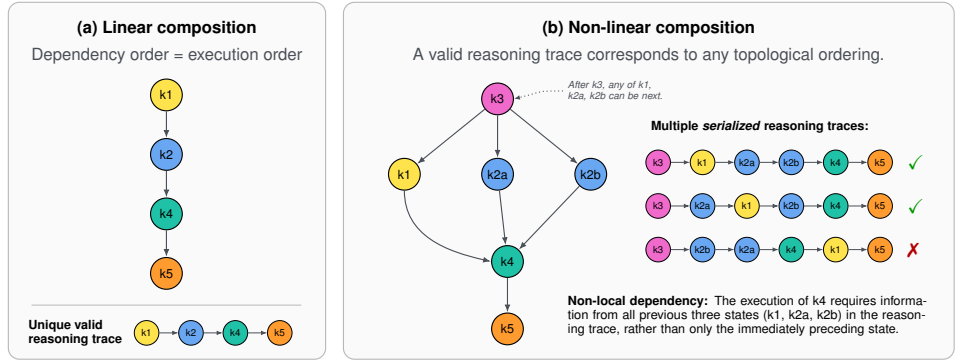
\begin{figure*}[t]
\vspace{-0.8em}
\centering
\resizebox{0.9\linewidth}{!}{%
\begin{tikzpicture}[
    font=\sffamily,
    >=Latex,
    node distance=1.1cm,
    skill/.style={
        circle,
        draw=black,
        line width=0.8pt,
        minimum size=8.0mm,
        inner sep=0pt,
        font=\sffamily\small
    },
    smallskill/.style={
        circle,
        draw=black,
        line width=0.55pt,
        minimum size=6.2mm,
        inner sep=0pt,
        font=\sffamily\scriptsize
    },
    arr/.style={
        -{Latex[length=2mm,width=1.5mm]},
        line width=0.8pt,
        draw=edgegray
    },
    smallarr/.style={
        -{Latex[length=1.5mm,width=1.1mm]},
        line width=0.6pt,
        draw=edgegray
    },
    note/.style={
        align=left,
        text=edgegray,
        font=\sffamily\itshape\scriptsize,
        text width=2.5cm
    }
]

% ============================================================
% (a) Linear composition panel
% ============================================================
\node[
    rounded corners=8pt,
    draw=panelborder,
    line width=0.8pt,
    fill=lightpanel,
    minimum width=8.0cm,
    minimum height=9.0cm,
    anchor=north west
] (panelA) at (0,0) {};

\node[
    font=\sffamily\bfseries\large,
    anchor=north
] at ($(panelA.north)+(0,-0.35)$)
{(a) Linear composition};

\node[
    font=\sffamily\large,
    text=edgegray,
    anchor=north
] at ($(panelA.north)+(0,-1.05)$)
{Dependency order = execution order};

% Linear chain
\node[skill, fill=kone]  (a-k1) at (4.0,-2.35) {k1};
\node[skill, fill=ktwo]  (a-k2) at (4.0,-3.85) {k2};
\node[skill, fill=kfour] (a-k4) at (4.0,-5.35) {k4};
\node[skill, fill=kfive] (a-k5) at (4.0,-6.85) {k5};

\draw[arr] (a-k1) -- (a-k2);
\draw[arr] (a-k2) -- (a-k4);
\draw[arr] (a-k4) -- (a-k5);

% Divider
\draw[draw=panelborder, line width=0.5pt] (0.6,-7.65) -- (7.4,-7.65);

% Unique traversal row
\node[
    anchor=west,
    font=\sffamily\bfseries\small,
    align=left
] at (0.4,-8.35)
{Unique valid\\reasoning trace};

\node[smallskill, fill=kone]  (a-r1) at (3.5,-8.35) {k1};
\node[smallskill, fill=ktwo,  right=0.48cm of a-r1] (a-r2) {k2};
\node[smallskill, fill=kfour, right=0.48cm of a-r2] (a-r3) {k4};
\node[smallskill, fill=kfive, right=0.48cm of a-r3] (a-r4) {k5};

\draw[smallarr] (a-r1) -- (a-r2);
\draw[smallarr] (a-r2) -- (a-r3);
\draw[smallarr] (a-r3) -- (a-r4);

% Linear explanation
% \node[
%     anchor=north west,
%     align=left,
%     text width=6.7cm,
%     font=\sffamily\small
% ] at (0.4,-9.05)
% {The dependency graph is a chain, so it fixes a unique correct reasoning trace.};

% ============================================================
% (b) Branch-and-merge composition panel
% ============================================================
\node[
    rounded corners=8pt,
    draw=panelborder,
    line width=0.8pt,
    fill=lightpanel,
    minimum width=15.4cm,
    minimum height=9.0cm,
    anchor=north west
] (panelB) at (8.45,0) {};

\node[
    font=\sffamily\bfseries\large,
    anchor=north
] at ($(panelB.north)+(0,-0.35)$)
{(b) Non-linear composition};

\node[
    font=\sffamily\large,
    text=edgegray,
    anchor=north
] at ($(panelB.north)+(0,-1.05)$)
{A valid reasoning trace corresponds to any topological ordering.};

% Branch-and-merge DAG
\node[skill, fill=kthree] (k3)  at (12.35,-2.45) {k3};
\node[skill, fill=kone]   (k1)  at (10.00,-4.35) {k1};
\node[skill, fill=ktwo]   (k2a) at (12.35,-4.35) {k2a};
\node[skill, fill=ktwo]   (k2b) at (14.70,-4.35) {k2b};
\node[skill, fill=kfour]  (k4)  at (12.60,-6.55) {k4};
\node[skill, fill=kfive]  (k5)  at (12.60,-8.25) {k5};

\draw[arr] (k3) -- (k1);
\draw[arr] (k3) -- (k2a);
\draw[arr] (k3) -- (k2b);

\draw[arr] (k2a) -- (k4);
\draw[arr] (k2b) -- (k4);

\draw[arr] (k1.south) .. controls +(0,-1.15) and +(-1.1,0.25) .. (k4.west);

\draw[arr] (k4) -- (k5);

% Annotation near branch point
\node[note, anchor=west] (readynote) at (14.00,-2.25)
{After k3, any of k1, \\k2a, k2b can be next.};

\draw[
    -{Latex[length=1.5mm,width=1.2mm]},
    dotted,
    line width=0.75pt,
    draw=edgegray
] (readynote.west) .. controls +(-0.7,0.0) and +(0.8,0.2) .. (k3.east);

% Right topological traversal examples
\node[
    font=\sffamily\bfseries\small,
    anchor=west
] at (16.10,-3.15)
{Multiple \emph{serialized} reasoning traces:};

% Row 1
\node[smallskill, fill=kthree] (r1a) at (16.40,-4.05) {k3};
\node[smallskill, fill=kone,   right=0.48cm of r1a] (r1b) {k1};
\node[smallskill, fill=ktwo,   right=0.48cm of r1b] (r1c) {k2a};
\node[smallskill, fill=ktwo,   right=0.48cm of r1c] (r1d) {k2b};
\node[smallskill, fill=kfour,  right=0.48cm of r1d] (r1e) {k4};
\node[smallskill, fill=kfive,  right=0.48cm of r1e] (r1f) {k5};
\node[font=\sffamily\bfseries\Large, text=green!60!black, right=0.25cm of r1f] {$\checkmark$};

\draw[smallarr] (r1a) -- (r1b);
\draw[smallarr] (r1b) -- (r1c);
\draw[smallarr] (r1c) -- (r1d);
\draw[smallarr] (r1d) -- (r1e);
\draw[smallarr] (r1e) -- (r1f);

% Row 2
\node[smallskill, fill=kthree] (r2a) at (16.40,-5.15) {k3};
\node[smallskill, fill=ktwo,   right=0.48cm of r2a] (r2b) {k2a};
\node[smallskill, fill=kone,   right=0.48cm of r2b] (r2c) {k1};
\node[smallskill, fill=ktwo,   right=0.48cm of r2c] (r2d) {k2b};
\node[smallskill, fill=kfour,  right=0.48cm of r2d] (r2e) {k4};
\node[smallskill, fill=kfive,  right=0.48cm of r2e] (r2f) {k5};
\node[font=\sffamily\bfseries\Large, text=green!60!black, right=0.25cm of r2f] {$\checkmark$};

\draw[smallarr] (r2a) -- (r2b);
\draw[smallarr] (r2b) -- (r2c);
\draw[smallarr] (r2c) -- (r2d);
\draw[smallarr] (r2d) -- (r2e);
\draw[smallarr] (r2e) -- (r2f);

% Row 3
\node[smallskill, fill=kthree] (r3a) at (16.40,-6.25) {k3};
\node[smallskill, fill=ktwo,   right=0.48cm of r3a] (r3b) {k2b};
\node[smallskill, fill=ktwo,   right=0.48cm of r3b] (r3c) {k2a};
\node[smallskill, fill=kfour,  right=0.48cm of r3c] (r3d) {k4};
\node[smallskill, fill=kone,   right=0.48cm of r3d] (r3e) {k1};
\node[smallskill, fill=kfive,  right=0.48cm of r3e] (r3f) {k5};
\node[
    font=\sffamily\bfseries\Large,
    text=red!70!black,
    right=0.25cm of r3f
] {\ding{55}};

\draw[smallarr] (r3a) -- (r3b);
\draw[smallarr] (r3b) -- (r3c);
\draw[smallarr] (r3c) -- (r3d);
\draw[smallarr] (r3d) -- (r3e);
\draw[smallarr] (r3e) -- (r3f);

% Branch-and-merge explanation
\node[
    anchor=north west,
    align=left,
    text width=9cm,
    font=\sffamily\small
] at (14.00,-7.3)
{\textbf{Non-local dependency:} The execution of k4 requires information from all previous three states (k1, k2a, k2b) in the reasoning trace, rather than only the immediately preceding state.};

\end{tikzpicture}%
}
\caption{\textbf{Reasoning traces serialize dependency graphs.}
An explicit step-by-step solution can be modeled as a topological ordering of the task dependency graph.}
\label{fig:linear-nonlinear}
\vspace{-1.0em}
\end{figure*}

\paragraph{Linear vs. non-linear compositions: non-local dependency.} We model an explicit step-by-step solution as a sequential \emph{reasoning trace} that serializes the dependency graph $G_\tau$. Under this abstraction, any valid trace must respect the dependencies in $G_\tau$, and thus corresponds to a topological ordering. As shown in Figure~\ref{fig:linear-nonlinear}, linear compositions are chains, so the trace follows a unique order. In contrast, branch-merge graphs are non-linear: branches may be executed in different valid orders before being merged. At merge points, the model must retrieve and combine intermediate outputs from earlier branches, rather than only propagating the immediately preceding state, which therefore requires coordinating \emph{non-local} dependencies. 

\section{Controlled RL study of compositional reasoning}
\label{sec:empirical-compositionality}

We present our empirical contribution by systematically evaluating compositional reasoning under RL fine-tuning using the three-level framework above. We use data-structure tasks from DSR-Bench~\citep{he2026llmsreasonstructurallybenchmarking}, which are well-suited because they provide deterministic, exact verification for RL reward computation, along with clear skill decompositions and compositional structure. We first present our main data-design finding, the decomposed-to-composed asymmetry. We then analyze the sources of compositional difficulty across the three levels and study generalization beyond length extrapolation.

\begin{figure}[t]
\centering
\includegraphics[width=\textwidth]{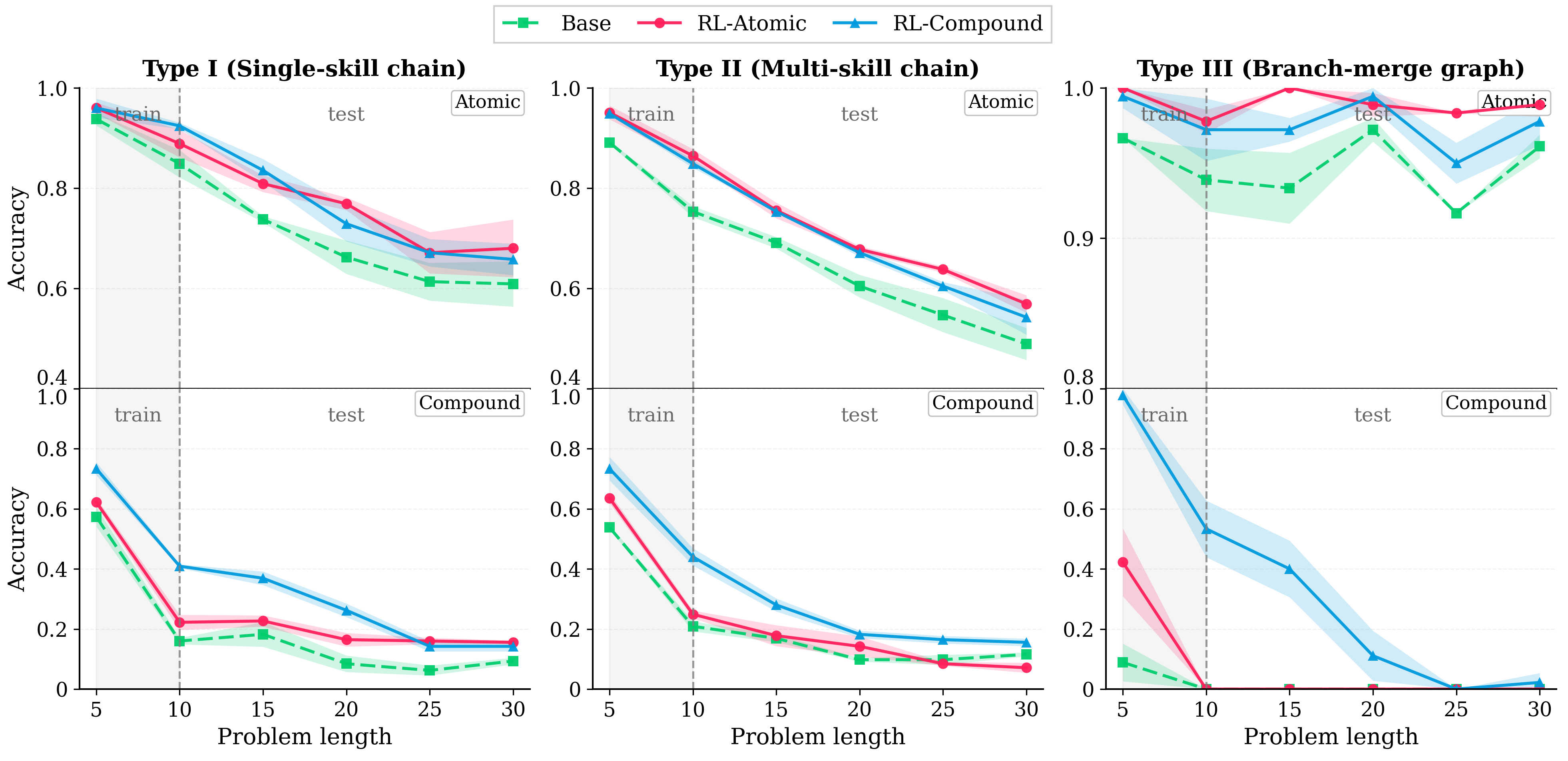}
\caption{Evaluation on three compositionality types with Qwen3-4B-Instruct, averaged across three runs with the standard deviation shown as a shaded region. Rows evaluate \emph{atomic} (decomposed skill) and \emph{compound} (composed task) tasks. RL-Atomic trains on atomic tasks, while RL-Compound trains on compound tasks. Models are trained on small problem lengths (5-10) and evaluated on larger ones (up to 30). Results show a consistent decomposed-to-composed asymmetry.}
\label{fig:decompose_compose}
\end{figure}

\subsection{Decomposed-skill vs. composed-task training}
\label{sec:decompose-vs-compose}

Our main experimental contrast compares RL post-training on decomposed skills and on composed tasks. We use five data-structure domains of varying difficulty: Array, Binary Search Tree (BST), Bloom Filter, Hashmap, and Heap. Task length denotes the initial input size for atomic tasks and the number of operations for compound tasks. Details on task description, example prompt, and data preparation can be found in Appendix~\ref{app:experiment-details}.

\paragraph{Atomic vs. compound tasks.}
An \emph{atomic} task corresponds to a decomposed skill: a single skill-invocation applied to a given state. For example, in an atomic array task of length 3, given $[1,2,3]$, the model may be asked to insert $4$ at position $2$, yielding $[1,2,4,3]$. A \emph{compound} task corresponds to a composed task: a sequence of skill invocations, where intermediate states must be carried forward. For example, in a compound array task of length 3, given $[1,2,3]$, the model may execute ``insert $4$ at position $2$, delete position $0$, then insert $5$ at position $1$,'' and return the final array.

\paragraph{Decomposed-skill vs. composed-task training.}
We use Qwen3-4B-Instruct~\citep{qwen3technicalreport} as the main \textbf{Base} model, and additionally evaluate two more models from different families and sizes: OLMo-3-7B-Instruct~\citep{olmo2025olmo3} and Llama-3.1-8B-Instruct~\citep{grattafiori2024llama3herdmodels} on multi-skill chains. We use instruct models for their instruction-following ability, while excluding models already heavily post-trained for reasoning. For each base model, we compare \textbf{RL-Atomic}, trained only on atomic tasks, with \textbf{RL-Compound}, trained only on compound tasks. Both use Group Relative Policy Optimization~\citep{deepseekai2025deepseekr1incentivizingreasoningcapability} with binary final rewards, while keeping the number of training steps the same. Each problem in DSR-Bench has a uniquely verifiable final answer \citep{he2026llmsreasonstructurallybenchmarking}. Refer to Appendix \ref{app:experiment-details} for details on reward computation, training setup, and hyperparameters.

\paragraph{Main observation: Decomposed-to-composed asymmetry.}
As shown in Figures~\ref{fig:decompose_compose} and~\ref{fig:olmo-llama-atomic-compound}, we observe a consistent asymmetry across three levels of compositionality and model families. Atomic training often improves atomic-task performance, but does not reliably transfer to compound tasks. In contrast, compound training largely preserves or improves atomic-task performance while substantially improving compound-task performance. This suggests that learning decomposed skills in isolation is insufficient for learning how to compose them: composed-task training exposes the model to both the skills and their interactions. We provide theoretical support in Section~\ref{sec:theory}. This shared observation motivates the remaining subsections, where we analyze sources of compositional challenges in a more fine-grained manner using the three-level framework introduced in Section~\ref{sec:compositionality-framework}.

\subsection{Level I (Single-skill chain): Horizon generalization}
\label{sec:single-skill}

Single-skill chained composition tests whether a model can repeatedly apply one skill invocation over a long dependent rollout. As shown in Figure~\ref{fig:three-compositionality}, the dependency graph is a path with the same skill type at every node, yielding a unique valid serialized reasoning trace. This setting isolates horizon generalization, where horizon is the number of sequential skill invocations. Note this is distinct from problem length, which denotes input size. Empirically, we instantiate this setting by stacking insertion operations across the five data-structure domains, treating insertion as a skill.

\paragraph{Compositional challenge 1: horizon generalization.}
In Level I (Single-skill chain) of Figure~\ref{fig:decompose_compose}, both RL-Atomic and RL-Compound consistently improve over the base model, indicating that RL post-training improves the underlying skill. However, the decomposed-to-composed asymmetry persists: RL-Atomic gives only marginal gains on compound tasks, whereas RL-Compound improves both atomic and compound performance. Moreover, performance drops much more sharply on compound than atomic tasks across all models, showing that even repeated invocation of a simple skill poses a horizon-generalization challenge: success requires staying correct across many intermediate states, where one early error can corrupt all later operations.

\subsection{Level II (Multi-skill chain): Level I + Skill switching} \label{sec:multi-skill}

\begin{figure}[h]
\centering
\includegraphics[width=0.95\linewidth]{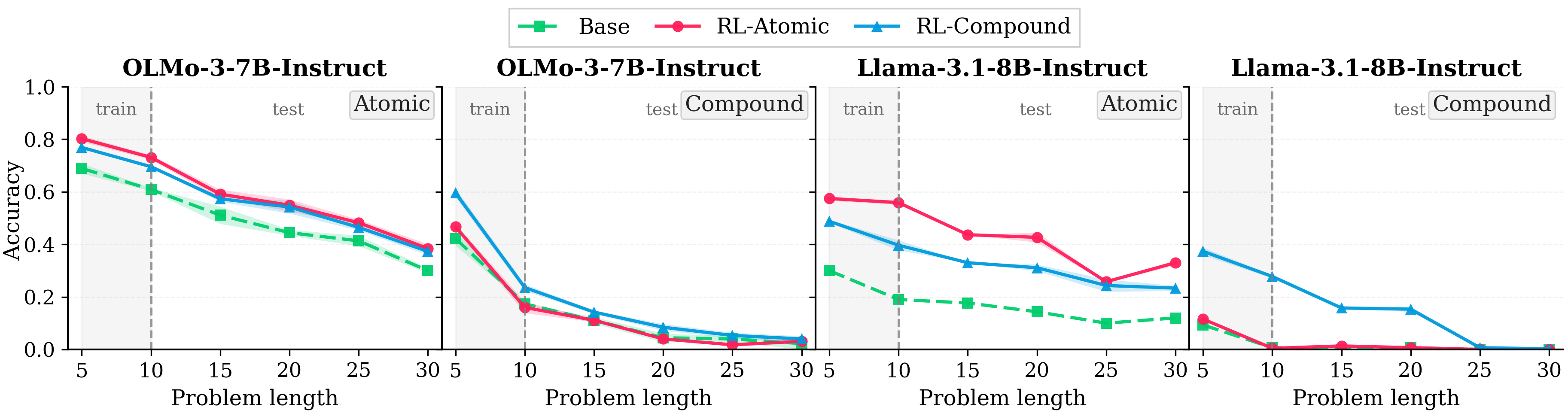}
\caption{Multi-skill chained composition with OLMo-3-7B-Instruct and Llama-3.1-8B-Instruct.}
\label{fig:olmo-llama-atomic-compound}
\end{figure}

Multi-skill chained composition adds a second bottleneck beyond horizon generalization: skill switching. Unlike single-skill chain, where one skill is repeatedly invoked, multi-skill chain requires composing heterogeneous skills in sequence and handling the interfaces between them, where the output produced by one skill must be converted into the appropriate input state for the next. We instantiate this setting by adding deletion as an additional skill across the five data-structure domains. Thus, RL-Atomic trains on single-step insertion or deletion, while RL-Compound trains on multi-step tasks with randomly interleaved insertions and deletions.

\paragraph{Compositional challenge 2: skill switching.}
The decomposed-to-composed asymmetry becomes more pronounced in Level II (Multi-skill chain) of Figure~\ref{fig:decompose_compose}, and also appears for OLMo-3-7B-Instruct and Llama-3.1-8B-Instruct in Figure~\ref{fig:olmo-llama-atomic-compound}. With heterogeneous skill composition, RL-Compound clearly improves compound-task performance, whereas RL-Atomic provides little or no gain over the base models. On atomic tasks, RL-Compound remains comparable to RL-Atomic, indicating that compound training improves on both atomic and compound tasks. Figure~\ref{fig:skill-switch} further shows that RL-Atomic stays close to the base model as the number of skill switches increases, suggesting that atomic training does not teach transitions between heterogeneous skills. Thus, decomposed-skill training becomes insufficient beyond repeatedly stacking one skill: the model must learn to switch skills and propagate intermediate states across changing skill input and output requirements.

\begin{figure}[h]
\centering

\begin{minipage}{0.38\linewidth}
    \centering
    \includegraphics[width=\linewidth]{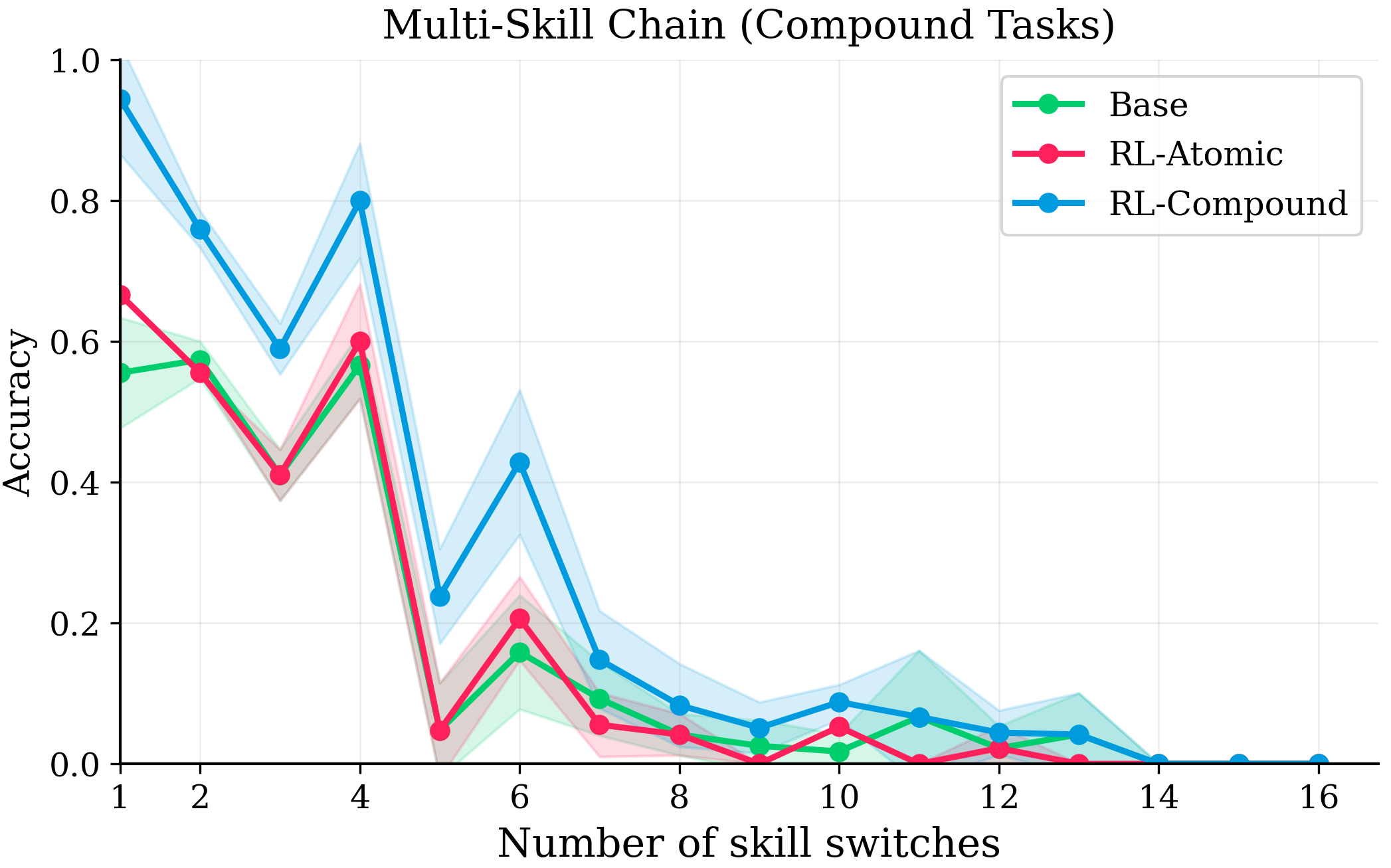}
    \caption{Accuracy as a function of the number of skill switches between insertion and deletion.}
    \label{fig:skill-switch}
\end{minipage}
\hfill
\begin{minipage}{0.6\linewidth}
    \centering
    \begin{minipage}{0.49\linewidth}
        \centering
        \includegraphics[width=\linewidth]{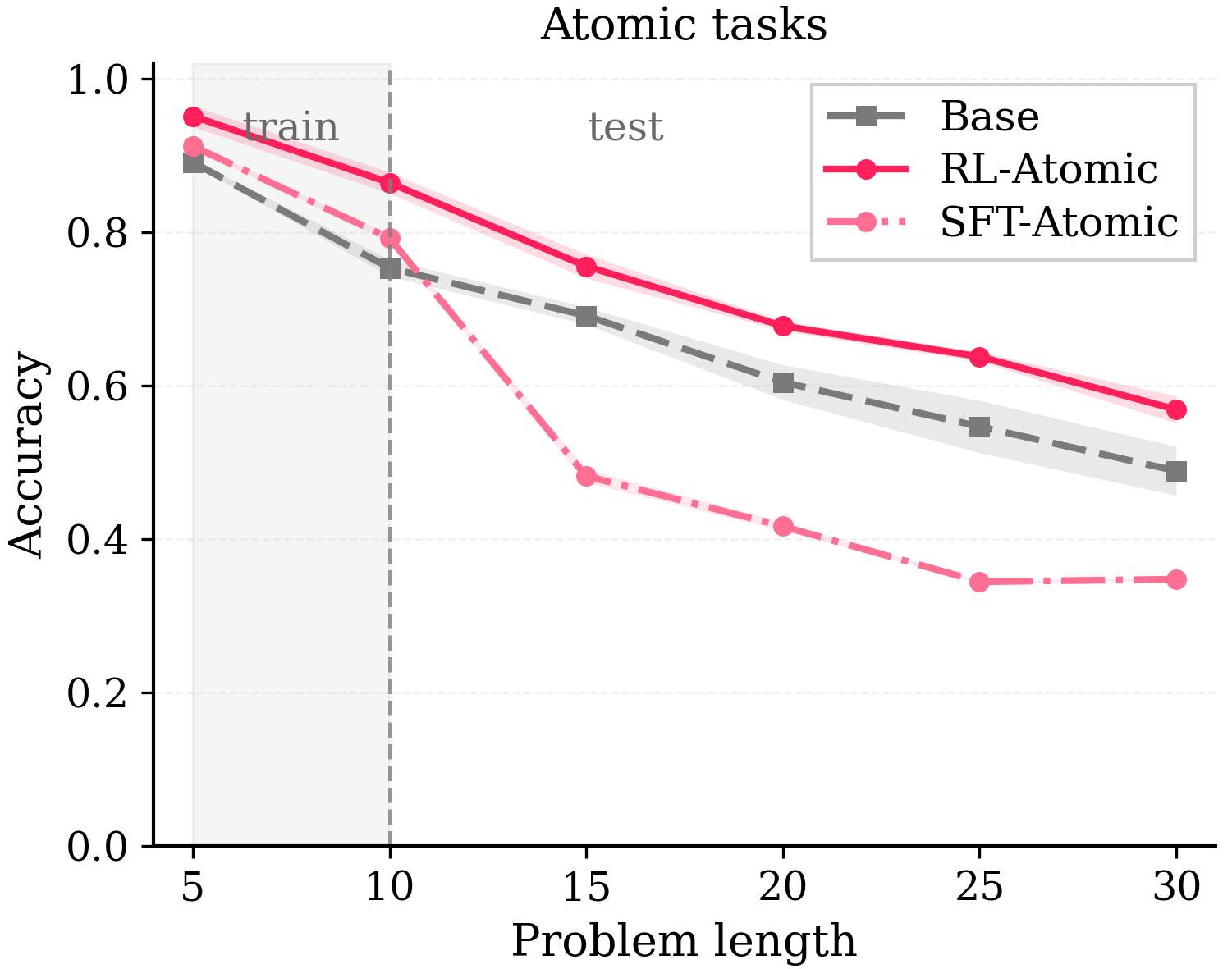}
    \end{minipage}
    \hfill
    \begin{minipage}{0.49\linewidth}
        \centering
        \includegraphics[width=\linewidth]{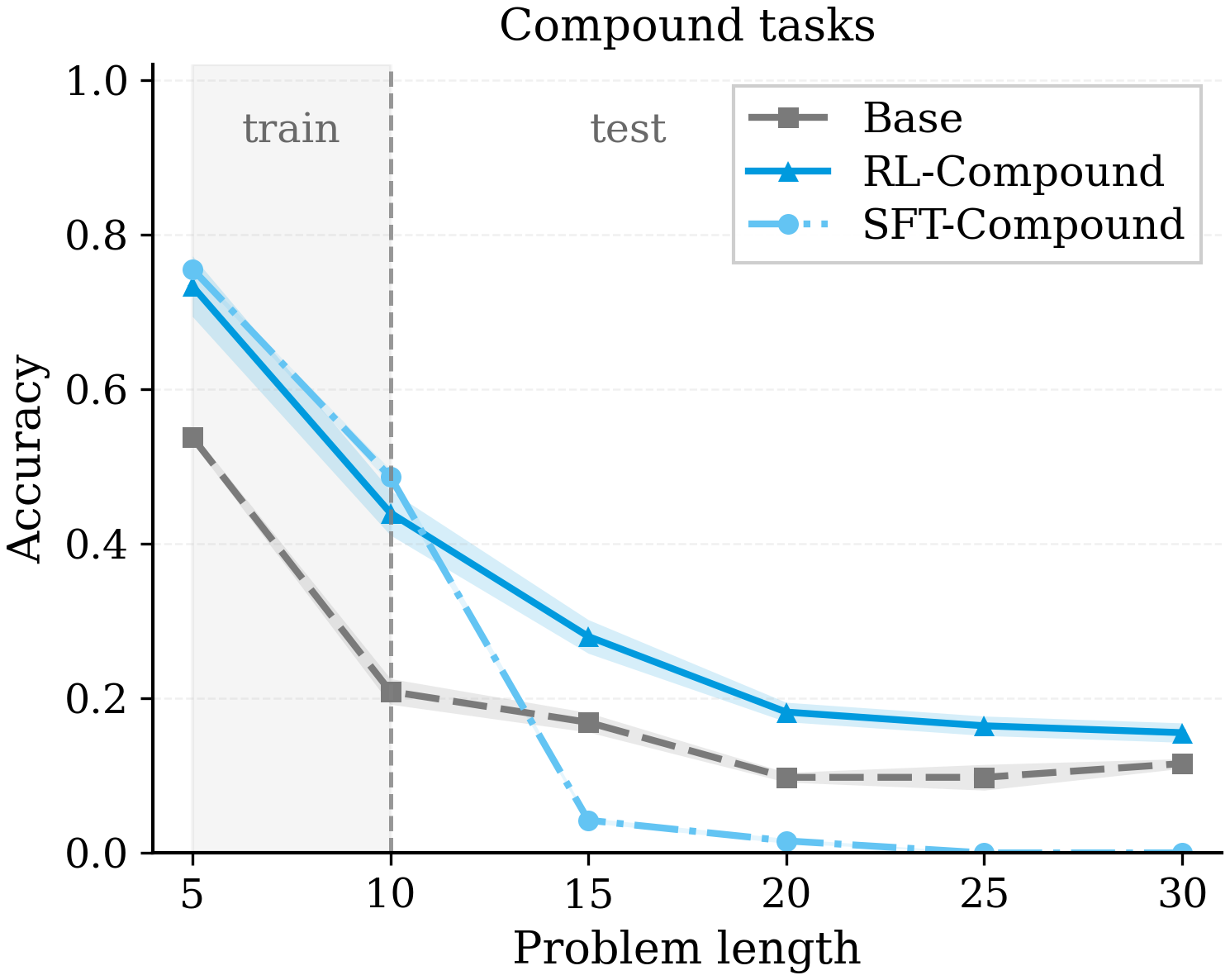}
    \end{minipage}

    \caption{Comparison of post-training methods: supervised fine-tuning (SFT) vs. reinforcement learning (RL) for multi-skill chained composition.}
    \label{fig:sft-rl}
\end{minipage}

\end{figure}

\paragraph{SFT vs. RL post-training.}
Both supervised fine-tuning (SFT) and reinforcement learning (RL) are widely used post-training methods. We further compare RL with supervised fine-tuning (SFT) using programmatically generated SFT training samples. Data-structure tasks make this feasible because their algorithmic structure allows chain-of-thought templates for each skill to be instantiated with task-specific values. In principle, these structured traces should make the training pattern easy for SFT to learn. However, as shown in Figure~\ref{fig:sft-rl}, SFT performs well in-domain but degrades sharply on larger out-of-domain instances, whereas RL shows better length generalization. This suggests that SFT mainly learns the in-domain pattern, while RL better supports generalization beyond the training regime, consistent with prior findings~\citep{chu2025sft}. Details of SFT training are in Appendix \ref{app:experiment-details}.

\subsection{Level III (Branch-merge graph): Level II + Non-local dependency} \label{sec:dag-compositionality}

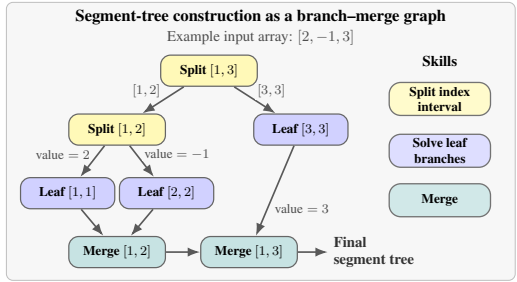
\begin{wrapfigure}{r}{0.5\linewidth}
\vspace{-3mm}
    \centering
    \resizebox{\linewidth}{!}{
    \resizebox{\linewidth}{!}{%
\begin{tikzpicture}[
    x=1cm, y=1cm, >=Latex,
    panel/.style={
        draw=black!25, fill=black!3, rounded corners=4pt,
        line width=0.5pt
    },
    box/.style={
        draw=black!60, rounded corners=6pt, line width=0.75pt,
        minimum width=1.82cm, minimum height=0.66cm,
        inner sep=2pt, align=center, font=\footnotesize
    },
    split/.style={box, fill=yellow!35, minimum width=2.05cm},
    leaf/.style={box, fill=blue!18, minimum width=2.00cm},
    merge/.style={box, fill=teal!22, minimum width=2.05cm},
    stepnode/.style={
        draw=black!55, rounded corners=7pt, line width=0.8pt,
        minimum width=2.15cm, minimum height=0.70cm,
        inner sep=2pt, align=center, font=\footnotesize\bfseries
    },
    stepa/.style={stepnode, fill=yellow!25},
    stepb/.style={stepnode, fill=blue!13},
    stepc/.style={stepnode, fill=teal!18},
    dep/.style={->, line width=1.0pt, draw=black!65},
    faint/.style={->, densely dotted, line width=0.7pt, draw=black!45},
    callout/.style={
        font=\footnotesize\bfseries, text=black!70,
        fill=white, inner sep=1.5pt, rounded corners=2pt
    },
    edgelab/.style={
        font=\footnotesize, text=black!70,
        fill=black!3, inner sep=1pt
    },
    title/.style={font=\bfseries\normalsize, align=center}
]

% ------------------------------------------------------------------
% Bounding box and outer panel
% ------------------------------------------------------------------
\path[use as bounding box] (0,0) rectangle (10.8,-5.9);
\draw[panel] (0,0) rectangle (10.8,-5.9);

% ------------------------------------------------------------------
% Title
% ------------------------------------------------------------------
\node[title] at (5.40,-0.30)
    {Segment-tree construction as a branch--merge graph};
\node[font=\small, text=black!70] at (5.40,-0.74)
    {Example input array: $[2,-1,3]$};

% ------------------------------------------------------------------
% Right-side vertical step nodes
% ------------------------------------------------------------------
\node[font=\small\bfseries] at (9.20,-1.25) {Skills};
\node[stepa] (st1) at (9.20,-2.05) {Split index\\interval};
\node[stepb] (st2) at (9.20,-3.15) {Solve leaf\\branches};
\node[stepc] (st3) at (9.20,-4.20) {Merge};

%\draw[dep] (st1) -- (st2);
%\draw[dep] (st2) -- (st3);

% ------------------------------------------------------------------
% Main branch-merge graph
% ------------------------------------------------------------------
\node[split] (S13) at (4.30,-1.45) {\textbf{Split} $[1,3]$};

\node[split] (S12) at (2.35,-2.70) {\textbf{Split} $[1,2]$};
\node[leaf]  (L33) at (6.25,-2.70) {\textbf{Leaf} $[3,3]$};

\node[leaf]  (L11) at (1.30,-4.05) {\textbf{Leaf} $[1,1]$};
\node[leaf]  (L22) at (3.40,-4.05) {\textbf{Leaf} $[2,2]$};

\node[merge] (M12) at (2.35,-5.30) {\textbf{Merge} $[1,2]$};
\node[merge] (M13) at (5.15,-5.30) {\textbf{Merge} $[1,3]$};

% Final output as text
\node[font=\small\bfseries, align=left, text=black!80]
    (OUT) at (7.8,-5.30) {Final\\segment tree};

% ------------------------------------------------------------------
% Edges
% ------------------------------------------------------------------
\draw[dep] (S13) -- node[edgelab, above left] {$[1,2]$} (S12);
\draw[dep] (S13) -- node[edgelab, above right] {$[3,3]$} (L33);

\draw[dep] (S12) -- node[edgelab, above left] {$\text{value}=2$} (L11);
\draw[dep] (S12) -- node[edgelab, above right] {$\text{value}=-1$} (L22);

\draw[dep] (L11) -- (M12);
\draw[dep] (L22) -- (M12);

\draw[dep] (M12) -- (M13);
\draw[dep] (L33) --
    node[edgelab, right, pos=0.7, xshift=3pt] {$\text{value}=3$}
    (M13);

\draw[dep] (M13) -- (OUT);

% ------------------------------------------------------------------
% Conceptual callouts
% ------------------------------------------------------------------
% \node[callout, fill=yellow!12] at (5.15,-2.05) {branch};
% \draw[faint] (4.95,-2.08)
%     .. controls (4.48,-2.22) and (4.18,-2.38) ..
%     (4.00,-2.60);
% \draw[faint] (5.35,-2.08)
%     .. controls (5.65,-2.22) and (5.80,-2.38) ..
%     (5.88,-2.60);

% \node[callout, fill=blue!8] at (4.25,-4.10) {independent branches};

% \node[callout, fill=teal!10] at (3.15,-4.85) {merge};
% \draw[faint] (2.92,-4.87) -- (2.30,-5.02);

% \node[callout, fill=teal!10] at (5.75,-4.72)
%     {combine earlier results};

\end{tikzpicture}%
}
    }
\caption{A simplified example illustration.}
\label{fig:segment-branch-merge}
\vspace{-2mm}
\end{wrapfigure}

As contrasted in Section \ref{sec:compositionality-framework}, unlike the linear compositions above, branch--merge graphs represent non-linear composition: computation branches into multiple intermediate chains and later merges their results. A valid reasoning trace may correspond to any topological ordering of the dependency graph, introducing a third bottleneck beyond horizon generalization and skill switching: non-local dependency coordination. At merge points, the model must retrieve, align, and combine outputs from earlier branches, rather than simply propagate the immediately preceding state. We use segment-tree construction as the branch--merge task, illustrated in Figure~\ref{fig:segment-branch-merge}, with details deferred to Appendix~\ref{app:experiment-details}. 

\paragraph{Compositional challenge 3: non-local dependency.}
The decomposed-to-composed asymmetry is most significant in Level III (Branch--merge graph) of Figure~\ref{fig:decompose_compose}. All methods remain near ceiling on atomic tasks, indicating that the underlying skills are not difficult. However, a huge gap emerges on compound tasks: RL-Atomic gives little improvement over the base model and quickly collapses to near-zero accuracy, whereas RL-Compound performs substantially better. This suggests that decomposed-skill training is insufficient when reasoning requires coordinating non-local dependencies across branches. We note RL-Compound also degrades sharply as problem length increases, showing that non-local coordination is a stronger bottleneck than long horizon or skill switching alone. 

\subsection{Beyond length: structural shift and unseen skills}

Above, we primarily vary problem length or execution horizon. We now test whether composed-task training generalizes under two additional forms of shift. First, Figure~\ref{fig:compositional-generalization}(a,b) tests robustness under structural distribution shift by varying BST skewness and the graph-construction threshold (which determines whether an edge is added between two nodes); composed-task training remains robust, though its gains depend on structural difficulty. Second, Figure~\ref{fig:compositional-generalization}(c) tests transfer to tasks requiring unseen skills: multi-dimensional computation in KD-Tree and Geometric Graph construction, and mapping natural-language scenarios to formal BST and graph tasks in Clinical Appointments and Galaxy Traveling. Performance improves only modestly, suggesting that composed training alone is insufficient when new reasoning skills are required. Experiment details can be found in Appendix~\ref{sec:generalization}.

\begin{figure}[h]
    \centering
    \includegraphics[width=\linewidth]{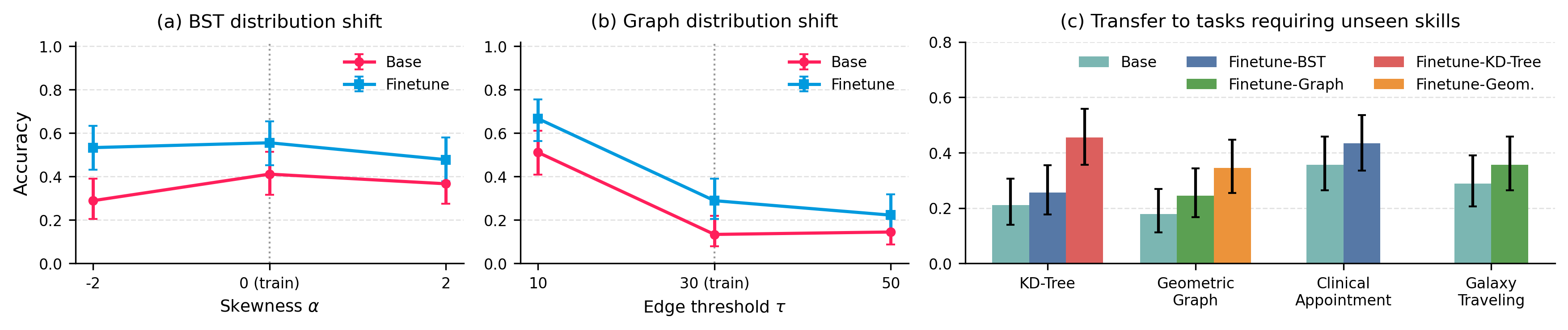}
    \caption{(a,b) Distribution-shift generalization on BST and graph construction. (c) Transfer to multi-dimensional and natural-language tasks. Error bars denote 95\% confidence intervals.}
    \label{fig:compositional-generalization}
\end{figure}

\section{Theoretical explanation of the asymmetry}
\label{sec:theory}

We provide a theoretical explanation of the underlying mechanisms for the observed decomposed-to-composed asymmetry. Full statements, proofs, and additional analysis are deferred to Appendix~\ref{sec:apx_proofs}.

\subsection{Why transfer from decomposition to composition is hard}
\label{sec:de-to-com-hard}

Following the compositionality framework in Section \ref{sec:compositionality-framework}, a task $\tau$ is represented by a dependency graph $G_{\tau} = (V_\tau, E_\tau)$. Each node $v \in V_\tau$ is a skill invocation with an input $x_v$ and an output $z_v$. Atomic tasks involve only one skill invocation; \emph{compound} tasks require multiple dependent invocations before returning the sink-node output $z_{v_{\text{sink}}}$. 

We model an autoregressive solution as a serialized rollout of $G_\tau$ (Figure \ref{fig:linear-nonlinear}). For a \emph{compound} task with horizon $|V_\tau| = T$, the step-$t$ state $s_t$ encodes the task, the executed-node set $P(s_t) \subset V_\tau$, and all intermediate outputs so far. The set of nodes that are ready to execute next is $R(s_t) = \{ v \in V_{\tau} \setminus P(s_t) : \mathrm{pred}(v) \subseteq P(s_t) \}$. Correctness is defined over a set of next invocation-output pairs because branch--merge graphs may admit several valid topological orderings. Define the \emph{decision context} 
$Q(s_t)=\{(\kappa(v),x_v(s_t)):v\in R(s_t)\}$, where $x_v(s_t)$ is the input to ready node
$v$. Identifying invocations by skill type and input rather than node name, the set of
valid invocation-output pairs is

 \[
 \mathcal I^* (s_t) = \{((\kappa,x),f_\kappa(x)):(\kappa,x)\in
Q(s_t)\}.
 \]
Therefore, a policy $\hat{\pi}$ makes a mistake when $\hat{\pi}(s)\notin \mathcal I^*(s)$. Each valid step completes one ready node, so a rollout with no mistakes in $T$
steps produces the correct final output. For stochastic policies, we assume the output
distribution depends on the rollout history only through the current state; probabilities
include policy randomness.

Let $q$ be the decision-context distribution induced by the \emph{atomic} task training data. Define the atomic-distribution error as
$\epsilon_{\mathrm{atom}}(\hat\pi;q) = 
\mathbb E_{c\sim
q}[e_{\hat\pi}(c)]$.

For compound tasks $\tau \sim \rho_T$, let $\mu_t$ be the distribution of $Q(s_t)$ conditioned on there being no earlier mistakes. Assuming every context in the support of $\mu_t$ has nonzero probability under $q$ for every $t\in[T]$, define the coverage factor
\[
C_T(q)=\sup_{c:q(c)>0}\frac{\sum_{t=1}^T\mu_t(c)}{q(c)}<\infty,
\label{eq:main-coverage-factor}
\]
which measures how well the training data covers contexts encountered across all steps of a compound rollout. It is large when a correct rollout contains contexts that are rare under atomic supervision.

\begin{lemma}[informal]
\label{lem:main-coverage}
Under the preceding assumptions, let $F=1$ denote the event that the model fails to produce the correct final output $z_{v_{\text{sink}}}$ on a compound task drawn from $\rho_T$. Then, 
\begin{equation}
\Pr[F=1]
\le
C_T(q) \,\epsilon_{\mathrm{atom}}(\hat\pi;q).
\label{eq:main-coverage-bound}
\end{equation}
Moreover, this bound is tight in the worst case.
\end{lemma}

The bound in Equation~\ref{eq:main-coverage-bound} shows that decomposed-to-composed transfer requires more than low atomic error. The bound accumulates step errors across all $T$
rollout steps, under decision-context distributions that atomic training may not cover well. Although the bound has no explicit multiplicative factor $T$, it is not horizon-free: whenever the policy can reach every rollout position with positive mistake-free probability, the cumulative coverage factor must satisfy
$C_T(q)\ge T$. Thus, matching $q(c)$ to $T^{-1}\sum_{t=1}^T\mu_t(c)$ recovers an upper bound of $T\epsilon_{\mathrm{atom}}(\hat\pi;q)$, while a poorer coverage weakens this guarantee. Proposition~\ref{prop:no-horizon-free-transfer} in Appendix~\ref{sec:apx_b1} shows that this horizon dependence is unavoidable.

Lemma~\ref{lem:main-coverage} also helps
interpret differences across the three levels of compositionality. Level~I increases the rollout horizon, so even repeatedly applying a single skill leads to a $T\epsilon_{\mathrm{atom}}(\hat\pi;q)$-scaling in the worse case. Level~II introduces skill switches, so intermediate outputs must be carried across states around transitions between skill types. These
transitions can produce skill-input pairs that are rare under $q$, increasing $C_T(q)$, consistent with Figure~\ref{fig:skill-switch}. Level~III adds merge points whose correct execution depends on multiple predecessor outputs, and the
resulting input tuples may be rare under $q$. Level III can also expose several distinct ready skill-input pairs simultaneously, creating non-singleton contexts that lie outside the support of any purely atomic training distribution. When reached with positive probability before the first mistake, these contexts
create a structural gap in coverage: atomic error places no constraint on the policy's error at these decisions. Thus, even perfect atomic accuracy cannot rule out compound failure without additional assumptions linking performance on singleton and multi-invocation contexts.

\subsection{Why transfer from composition to decomposition is easier}
\label{sec:co-to-de-easy}

Performance generalization from compound to atomic tasks avoids accumulating errors across the rollout. Assume each compound task has a unique source node, so its initial decision context is a singleton. Therefore, under a non-recoverability assumption, good compound performance directly constrains the model's error on the first-step context induced by the compound-task distribution. 

Let $\epsilon_{\mathrm{comp}}(\hat{\pi}; \rho_T)$ 
be the compound-task error, and let $s_1$ be the initial state for $\tau \sim \rho_T$. The first-step error induced by the compound task distribution $\rho_T$ is $\epsilon_{\mathrm{atom}}^{(1)}(\hat\pi;\rho_T)= \Pr[\hat\pi(s_1)\notin\mathcal
I^*(s_1)]$. 
If a mistake in the first step necessarily prevents successful completion, then
$\epsilon_{\mathrm{atom}}^{(1)}(\hat\pi;\rho_T)
\le
\epsilon_{\mathrm{comp}}(\hat\pi;\rho_T)$
because every first-step error is a compound failure. 

For atomic test cases drawn from a target context distribution $q$ covered by this first-step marginal, the same argument holds up to a coverage factor under
context invariance.
Let $q_1^{\rho_T}$ be the distribution over first-step decision
contexts induced by compound tasks, i.e.,
$q_1^{\rho_T}(c)=\Pr_{\tau\sim\rho_T}[Q(s_1)=c]$. Assuming $q$ puts no mass on contexts outside the support of $q_1^{\rho_T}$ and the coverage factor below is finite, we have
\begin{equation*}
\epsilon_{\mathrm{atom}}(\hat\pi;q)
\le
C_1(q,{\rho_T})\,
\epsilon_{\mathrm{comp}}(\hat\pi;{\rho_T}),
\qquad
\text{where }C_1(q,{\rho_T})
=
\sup_{c:q_1^{\rho_T}(c)>0}\frac{q(c)}{q_1^{\rho_T}(c)}.
\label{eq:main-reverse-coverage}
\end{equation*}

When $q=q_1^{\rho_T}$, $C_1(q,{\rho_T})=1$, and small compound error implies small atomic error on this context distribution. Thus, under these assumptions, good compound performance controls error on covered atomic test cases,
consistent with the RL-Compound results in Figures~\ref{fig:decompose_compose} and~\ref{fig:olmo-llama-atomic-compound}.

\section{Pilot study: multi-call tool use}
\label{sec:pilot}

\begin{wrapfigure}{r}{0.45\linewidth}
\vspace{-1em}
\centering
\includegraphics[width=\linewidth]{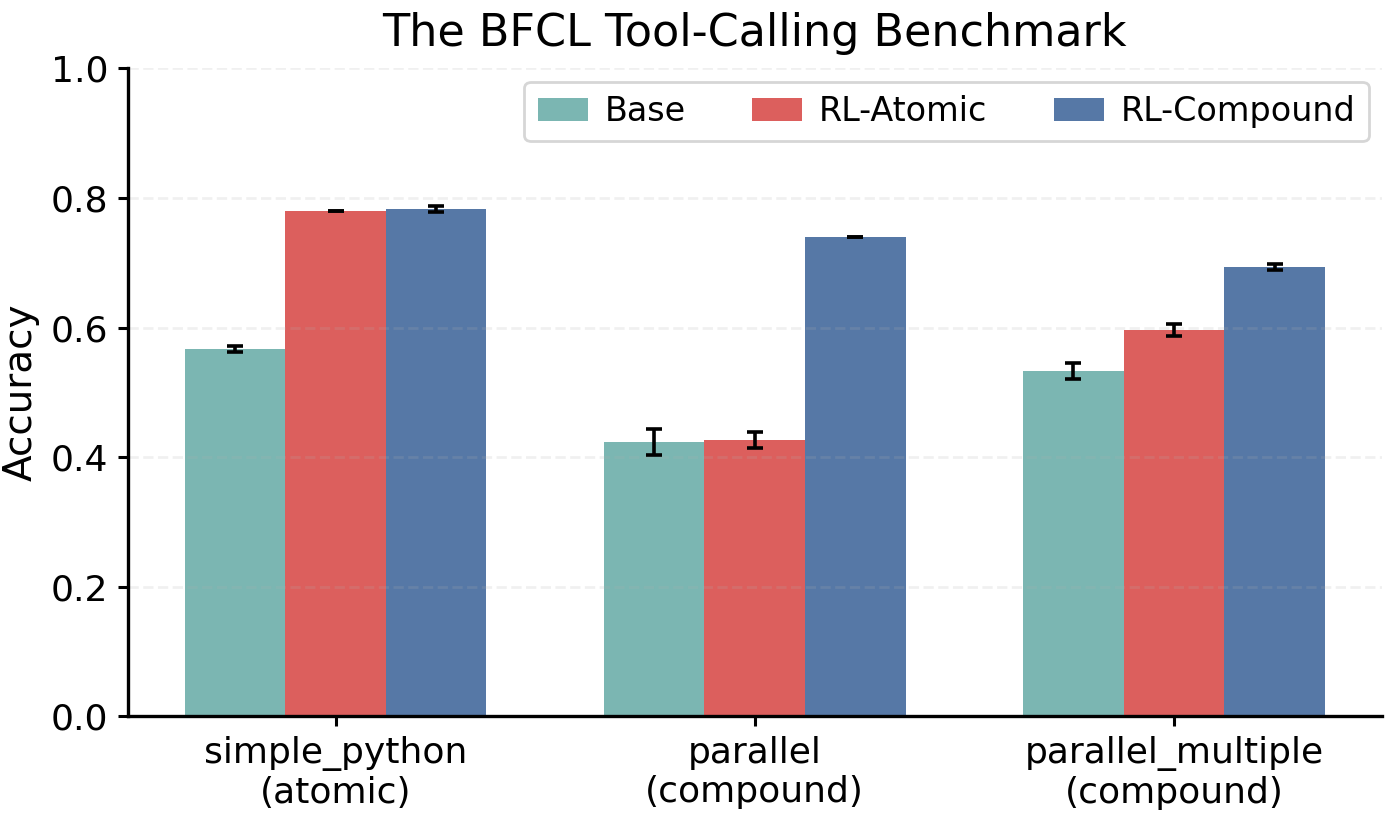}
\caption{Decomposed-to-composed asymmetry in a tool-calling benchmark.}
\label{fig:bfcl}
\vspace{-1em}
\end{wrapfigure}

We show preliminary evidence of practical implications for RL post-training data design as language models are increasingly applied to complex real-world tasks. We provide a pilot test using the Berkeley Function Calling Leaderboard (BFCL) \citep{patil2025bfcl}, a benchmark that maps user requests to valid tool calls. Here, a skill corresponds to a tool-calling schema for a specific function. We define atomic tasks using the \texttt{simple\_python} split, where each example requires one function call, and compound tasks using the \texttt{parallel} and \texttt{parallel\_multiple} splits, where each request requires multiple calls. We use the same RL setup as main experiments with Qwen3-4B-Instruct. Figure~\ref{fig:bfcl} shows the same asymmetry: decomposed-skill training does not generalize well to compound tasks, whereas composed-task training performs well on both tasks. The results further motivate composed-task training as a necessary and beneficial data-design principle for RL post-training. Dataset description and experiment details can be found in Appendix~\ref{app:experiment-details}. We discuss potential implications of this design principle for proof writing, coding agents, and tool-calling agents in Appendix~\ref{sec:practical-guide-data}.

\section{Conclusion}\label{sec:conclusion}
We study compositional reasoning in language models under RL post-training, formalizing it as the execution of reusable skills over dependency graphs with three levels of compositional complexity. Using controlled data-structure tasks, we evaluate these levels and find a consistent decomposed-to-composed asymmetry: training on decomposed skills transfers poorly to composed tasks, whereas training on composed tasks transfers more readily to decomposed settings. We explain this asymmetry through error compounding and coverage shift, and further evaluate generalization beyond length extrapolation under structural distribution shifts and transfer to tasks requiring unseen skills. A pilot study in multi-call tool use provides preliminary evidence that the same asymmetry can arise in a practical domain. Overall, our results suggest that RL post-training cannot rely on decomposed skill supervision alone: exposing models to structured interactions among skills is important for improving composed reasoning while preserving performance on decomposed skills.

\paragraph{Limitations and future work.}
We primarily evaluate compound tasks using final-output accuracy. Although intermediate-state accuracy could be informative, our preliminary studies found that requiring intermediate outputs reduced final accuracy, and potentially constrained models' ability to reason freely. Similarly, using Structured Outputs with JSON also reduced accuracy, likely due to added formatting burden. These design choices are further discussed in Appendix \ref{app:structured-output}. Future work could use LLM-based judging to track intermediate reasoning traces and parse final outputs. In addition, our empirical evaluation focuses on controlled data-structure tasks and broader validation on real-world benchmarks remains an important direction, with our BFCL pilot study as a first step.

\section*{Contribution statement}
YH led the project and contributed to the conceptual development, empirical components, and writing of the manuscript. YL contributed to the theoretical proofs on asymmetry. YW contributed to the conceptual framing and research discussions. EV advised the project.

\begin{ack}
YH is supported by a Cubist PhD Fellowship. YL is supported by an Amazon AI Fellowship. EV gratefully acknowledges support from the National Science Foundation under award CCF-2338226 and the AI2050 program at Schmidt Sciences. We thank Connor Lawless, Anders Wikum, and Nikil Selvam for their feedback on this manuscript. 
\end{ack}

\bibliographystyle{plainnat}
\bibliography{references}

@misc{grattafiori2024llama3herdmodels,
      title={The Llama 3 Herd of Models}, 
      author={Aaron Grattafiori et al.},
      year={2024},
      eprint={2407.21783},
      archivePrefix={arXiv},
      primaryClass={cs.AI},
      url={https://arxiv.org/abs/2407.21783}, 
}

@inproceedings{
    zhao2024can,
    title={Can Models Learn Skill Composition from Examples?},
    author={Haoyu Zhao and Simran Kaur and Dingli Yu and Anirudh Goyal and Sanjeev Arora},
    booktitle={The Thirty-eighth Annual Conference on Neural Information Processing Systems},
    year={2024},
    url={https://openreview.net/forum?id=1sLdprsbmk}
}

@inproceedings{
    yu2024skillmix,
    title={{SKILL}-{MIX}: a Flexible and Expandable Family of Evaluations for {AI} Models},
    author={Dingli Yu and Simran Kaur and Arushi Gupta and Jonah Brown-Cohen and Anirudh Goyal and Sanjeev Arora},
    booktitle={The Twelfth International Conference on Learning Representations},
    year={2024},
    url={https://openreview.net/forum?id=Jf5gplvglq}
}

@inproceedings{
    li2025unveiling,
    title={Unveiling the Compositional Ability Gap in Vision-Language Reasoning Model},
    author={Tianle Li and Jihai Zhang and Yongming Rao and Yu Cheng},
    booktitle={The Thirty-ninth Annual Conference on Neural Information Processing Systems},
    year={2025},
    url={https://openreview.net/forum?id=J76cCYTJub}
}

@inproceedings{
    redhardt2025scaling,
    title={Scaling can lead to compositional generalization},
    author={Florian Redhardt and Yassir Akram and Simon Schug},
    booktitle={The Thirty-ninth Annual Conference on Neural Information Processing Systems},
    year={2025},
    url={https://openreview.net/forum?id=hZt0daVIZi}
}

@inproceedings{
    ye2025how,
    title={How do Transformers Learn Implicit Reasoning?},
    author={Jiaran Ye and Zijun Yao and Zhidian Huang and Liangming Pan and Jinxin Liu and Yushi Bai and Amy Xin and Liu Weichuan and Xiaoyin Che and Lei Hou and Juanzi Li},
    booktitle={The Thirty-ninth Annual Conference on Neural Information Processing Systems},
    year={2025},
    url={https://openreview.net/forum?id=19ygs48nOa}
}

@inproceedings{
    chang2026characterizing,
    title={Characterizing Pattern Matching and Its Limits on Compositional Task Structures},
    author={Hoyeon Chang and Jinho Park and Hanseul Cho and Sohee Yang and Miyoung Ko and Hyeonbin Hwang and Seungpil Won and Dohaeng Lee and Youbin Ahn and Minjoon Seo},
    booktitle={The Fourteenth International Conference on Learning Representations},
    year={2026},
    url={https://openreview.net/forum?id=VCjlm003WL}
}

@inproceedings{yang-etal-2024-exploring,
    title = "Exploring Compositional Generalization of Large Language Models",
    author = "Yang, Haoran  and
      Lu, Hongyuan  and
      Lam, Wai  and
      Cai, Deng",
    editor = "Cao, Yang (Trista)  and
      Papadimitriou, Isabel  and
      Ovalle, Anaelia  and
      Zampieri, Marcos  and
      Ferraro, Francis  and
      Swayamdipta, Swabha",
    booktitle = "Proceedings of the 2024 Conference of the North American Chapter of the Association for Computational Linguistics: Human Language Technologies (Volume 4: Student Research Workshop)",
    month = jun,
    year = "2024",
    address = "Mexico City, Mexico",
    publisher = "Association for Computational Linguistics",
    url = "https://aclanthology.org/2024.naacl-srw.3/",
    doi = "10.18653/v1/2024.naacl-srw.3",
    pages = "16--24"
}

@misc{ahuja2025provablelengthcompositionalgeneralization,
      title={On Provable Length and Compositional Generalization}, 
      author={Kartik Ahuja and Amin Mansouri},
      year={2025},
      eprint={2402.04875},
      archivePrefix={arXiv},
      primaryClass={cs.LG},
      url={https://arxiv.org/abs/2402.04875}, 
}

@inproceedings{
    ramesh2024compositional,
    title={Compositional Capabilities of Autoregressive Transformers: A Study on Synthetic, Interpretable Tasks},
    author={Rahul Ramesh and Ekdeep Singh Lubana and Mikail Khona and Robert P. Dick and Hidenori Tanaka},
    booktitle={Forty-first International Conference on Machine Learning},
    year={2024},
    url={https://openreview.net/forum?id=L1eJ3NKPCd}
}

@misc{qwen3technicalreport,
      title={Qwen3 Technical Report}, 
      author={Qwen Team},
      year={2025},
      eprint={2505.09388},
      archivePrefix={arXiv},
      primaryClass={cs.CL},
      url={https://arxiv.org/abs/2505.09388}, 
}

@misc{olmo2025olmo3,
title={Olmo 3},
author={Team Olmo and Allyson Ettinger et al.},
year={2025},
eprint={2512.13961},
archivePrefix={arXiv},
primaryClass={cs.CL},
url={https://arxiv.org/abs/2512.13961},
}

@inproceedings{
chu2025sft,
title={{SFT} Memorizes, {RL} Generalizes: A Comparative Study of Foundation Model Post-training},
author={Tianzhe Chu and Yuexiang Zhai and Jihan Yang and Shengbang Tong and Saining Xie and Dale Schuurmans and Quoc V Le and Sergey Levine and Yi Ma},
booktitle={Forty-second International Conference on Machine Learning},
year={2025},
url={https://openreview.net/forum?id=dYur3yabMj}
}

@misc{he2026llmsreasonstructurallybenchmarking,
      title={Can LLMs Reason Structurally? Benchmarking via the Lens of Data Structures}, 
      author={Yu He and Yingxi Li and Colin White and Ellen Vitercik},
      year={2026},
      eprint={2505.24069},
      archivePrefix={arXiv},
      primaryClass={cs.LG},
      url={https://arxiv.org/abs/2505.24069}, 
}

@InProceedings{pmlr-v80-lake18a,
  title = 	 {Generalization without Systematicity: On the Compositional Skills of Sequence-to-Sequence Recurrent Networks},
  author =       {Lake, Brenden and Baroni, Marco},
  booktitle = 	 {Proceedings of the 35th International Conference on Machine Learning},
  pages = 	 {2873--2882},
  year = 	 {2018},
  editor = 	 {Dy, Jennifer and Krause, Andreas},
  volume = 	 {80},
  series = 	 {Proceedings of Machine Learning Research},
  month = 	 {10--15 Jul},
  publisher =    {PMLR},
}

@inproceedings{
keysers2020measuring,
title={Measuring Compositional Generalization: A Comprehensive Method on Realistic Data},
author={Daniel Keysers and Nathanael Sch{\"a}rli and Nathan Scales and Hylke Buisman and Daniel Furrer and Sergii Kashubin and Nikola Momchev and Danila Sinopalnikov and Lukasz Stafiniak and Tibor Tihon and Dmitry Tsarkov and Xiao Wang and Marc van Zee and Olivier Bousquet},
booktitle={International Conference on Learning Representations},
year={2020},
}

@inproceedings{kim-linzen-2020-cogs,
    title = "{COGS}: A Compositional Generalization Challenge Based on Semantic Interpretation",
    author = "Kim, Najoung  and
      Linzen, Tal",
    editor = "Webber, Bonnie  and
      Cohn, Trevor  and
      He, Yulan  and
      Liu, Yang",
    booktitle = "Proceedings of the 2020 Conference on Empirical Methods in Natural Language Processing (EMNLP)",
    month = nov,
    year = "2020",
    address = "Online",
    publisher = "Association for Computational Linguistics",
    url = "https://aclanthology.org/2020.emnlp-main.731/",
    doi = "10.18653/v1/2020.emnlp-main.731",
    pages = "9087--9105",
}

@misc{shao2024deepseekmathpushinglimitsmathematical,
      title={DeepSeekMath: Pushing the Limits of Mathematical Reasoning in Open Language Models}, 
      author={Zhihong Shao and Peiyi Wang and Qihao Zhu and Runxin Xu and Junxiao Song and Xiao Bi and Haowei Zhang and Mingchuan Zhang and Y. K. Li and Y. Wu and Daya Guo},
      year={2024},
      eprint={2402.03300},
      archivePrefix={arXiv},
      primaryClass={cs.CL},
      url={https://arxiv.org/abs/2402.03300}, 
}

@misc{deepseekai2025deepseekr1incentivizingreasoningcapability,
      title={DeepSeek-R1: Incentivizing Reasoning Capability in LLMs via Reinforcement Learning}, 
      author={DeepSeek-AI},
      year={2025},
      eprint={2501.12948},
      archivePrefix={arXiv},
      primaryClass={cs.CL},
      url={https://arxiv.org/abs/2501.12948}, 
}

@inproceedings{
    patil2025bfcl,
    title={The Berkeley Function Calling Leaderboard ({BFCL}): From Tool Use to Agentic Evaluation of Large Language Models},
    author={Shishir G Patil and Huanzhi Mao and Fanjia Yan and Charlie Cheng-Jie Ji and Vishnu Suresh and Ion Stoica and Joseph E. Gonzalez},
    booktitle={Forty-second International Conference on Machine Learning},
    year={2025},
    url={https://openreview.net/forum?id=2GmDdhBdDk}
}

@inproceedings{
    guo2026g1,
    title={\${\textbackslash}texttt\{G1\}\$: Teaching {LLM}s to Reason on Graphs with Reinforcement Learning},
    author={Xiaojun Guo and Ang Li and Yifei Wang and Stefanie Jegelka and Yisen Wang},
    booktitle={The Thirty-ninth Annual Conference on Neural Information Processing Systems},
    year={2026},
    url={https://openreview.net/forum?id=Lq4nneD2xX}
}

@article{sheng2024verl,
  title   = {HybridFlow: A Flexible and Efficient RLHF Framework},
  author  = {Guangming Sheng and Chi Zhang and Zilingfeng Ye and Xibin Wu and Wang Zhang and Ru Zhang and Yanghua Peng and Haibin Lin and Chuan Wu},
  year    = {2024},
  journal = {arXiv preprint arXiv: 2409.19256}
}

@inproceedings{kwon2023vllm,
  title={Efficient Memory Management for Large Language Model Serving with PagedAttention},
  author={Woosuk Kwon and Zhuohan Li and Siyuan Zhuang and Ying Sheng and Lianmin Zheng and Cody Hao Yu and Joseph E. Gonzalez and Hao Zhang and Ion Stoica},
  booktitle={Proceedings of the ACM SIGOPS 29th Symposium on Operating Systems Principles},
  year={2023}
}

@article{yuan2025f,
  title={From $ f (x) $ and $ g (x) $ to $ f (g (x)) $: LLMs Learn New Skills in RL by Composing Old Ones},
  author={Yuan, Lifan and Chen, Weize and Zhang, Yuchen and Cui, Ganqu and Wang, Hanbin and You, Ziming and Ding, Ning and Liu, Zhiyuan and Sun, Maosong and Peng, Hao},
  journal={arXiv preprint arXiv:2509.25123},
  year={2025}
}

@article{xu2026composition,
  title={Composition-RL: Compose Your Verifiable Prompts for Reinforcement Learning of Large Language Models},
  author={Xu, Xin and Bai, Clive and Yang, Kai and Chen, Tianhao and Chen, Yangkun and Liu, Weijie and Chen, Hao and Wang, Yang and Yang, Saiyong and Yang, Can},
  journal={arXiv preprint arXiv:2602.12036},
  year={2026}
}

@article{dziri2023faith,
  title={Faith and fate: Limits of transformers on compositionality (2023)},
  author={Dziri, Nouha and Lu, Ximing and Sclar, Melanie and Li, Xiang Lorraine and Jiang, Liwei and Lin, Bill Yuchen and West, Peter and Bhagavatula, Chandra and Le Bras, Ronan and Hwang, Jena D and others},
  journal={arXiv preprint arXiv:2305.18654},
  volume={3},
  year={2023}
}

@article{yin2025learning,
  title={Learning Composable Chains-of-Thought},
  author={Yin, Fangcong and Liu, Zeyu Leo and Leqi, Liu and Ye, Xi and Durrett, Greg},
  journal={arXiv preprint arXiv:2505.22635},
  year={2025}
}

@inproceedings{
lambert2025tulu,
title={Tulu 3: Pushing Frontiers in Open Language Model Post-Training},
author={Nathan Lambert and Jacob Morrison and Valentina Pyatkin and Shengyi Huang and Hamish Ivison and Faeze Brahman and Lester James Validad Miranda and Alisa Liu and Nouha Dziri and Xinxi Lyu and Yuling Gu and Saumya Malik and Victoria Graf and Jena D. Hwang and Jiangjiang Yang and Ronan Le Bras and Oyvind Tafjord and Christopher Wilhelm and Luca Soldaini and Noah A. Smith and Yizhong Wang and Pradeep Dasigi and Hannaneh Hajishirzi},
booktitle={Second Conference on Language Modeling},
year={2025},
url={https://openreview.net/forum?id=i1uGbfHHpH}
}

@inproceedings{
yue2026does,
title={Does Reinforcement Learning Really Incentivize Reasoning Capacity in {LLM}s Beyond the Base Model?},
author={Yang Yue and Zhiqi Chen and Rui Lu and Andrew Zhao and Zhaokai Wang and Yang Yue and Shiji Song and Gao Huang},
booktitle={The Thirty-ninth Annual Conference on Neural Information Processing Systems},
year={2026},
url={https://openreview.net/forum?id=4OsgYD7em5}
}

%%%%%%%%%%%%%%%%%%%%%%%%%%%%%%%%%%%%%%%%%%%%%%%%%%%%%%%%%%%%

\newpage
\appendix

\part*{\Large Appendix}
\addcontentsline{toc}{part}{Appendix}

\etocsettocstyle{\section*{\contentsname}}{}
\setcounter{tocdepth}{2} 
\localtableofcontents

\newpage

\section{Additional related work}\label{sec:additional-related}

A line of work studies when and why compositional generalization emerges. \citet{redhardt2025scaling} show that, with sufficient coverage, scaling data and model size can induce generalization to unseen task compositions, and provide theoretical results for modular task families. \citet{ahuja2025provablelengthcompositionalgeneralization} give guarantees for length and compositional generalization under sufficient training diversity. In controlled synthetic settings, \citet{ramesh2024compositional} show that autoregressive transformers can generalize to many unseen compositions, while \citet{chang2026characterizing} characterize the limits of pattern matching on compositional task structures and show that coverage alone is not always sufficient. In symbolic reasoning, \citet{ye2025how} study multi-hop generalization and find that atomic supervision mainly accelerates learning rather than determining it. However, these analyses are developed largely in supervised or tightly controlled settings, and do not directly address how post-training changes compositional reasoning in modern language models. 

\section{Omitted Proofs}\label{sec:apx_proofs}

In this section, we provide omitted proofs and details from Section~\ref{sec:theory} of the paper.

\subsection{Omitted Proofs from Section~\ref{sec:de-to-com-hard}}\label{sec:apx_b1}

The key challenge in transferring from decomposed skills to composed tasks is that solving a compound task requires the trained policy to remain correct over a serialized rollout of the task dependency graph. Even if each individual skill invocation is easy under the atomic training distribution, errors can accumulate over the reasoning trace, and the compound rollout may visit intermediate states that are rare under atomic supervision.

\paragraph{Formal rollout model.}
Fix a horizon $T$ and let $\rho_T$ be a distribution over compound task instances whose dependency graphs have $|V_\tau|=T$. A task instance $\tau\sim\rho_T$ is random, while $T$ is fixed throughout this subsection. As in Section~\ref{sec:compositionality-framework}, write
\[
G_\tau=(V_\tau,E_\tau)
\]
for the dependency graph of $\tau$. Each node $v\in V_\tau$ is labeled by a skill type $\kappa(v)$ and local information $\ell_v$. Its input and correct output at a rollout state are defined below.

We write
\[
\mathrm{pred}(v)=\{u\in V_\tau:(u,v)\in E_\tau\}
\]
for the predecessors of $v$, and assume a unique sink node $v_{\mathrm{sink}}$ whose output is the final task output.

Let $\mathcal S$ be a countable rollout-state space. A state $s\in\mathcal S$ contains the task instance $\tau$, the completed-node set
\[
P(s)\subseteq V_\tau,
\]
and the intermediate outputs produced so far, written as $\{z_u(s):u\in P(s)\}$. Thus the state contains all task-relevant context needed to determine which skill invocations are available next and what their correct outputs are. The nodes that are ready to execute at state $s$ are
\[
R(s)=\{v\in V_\tau\setminus P(s):\mathrm{pred}(v)\subseteq P(s)\}.
\]
For each ready node $v\in R(s)$, its input is formed from the task-local information and predecessor outputs by
\[
x_v(s)
=
\psi_v\!\left(\ell_v,\left(z_u(s)\right)_{u\in\mathrm{pred}(v)}\right)
\in \mathcal X_{\kappa(v)},
\]
and its correct output is
\[
z^*_v(s)=f_{\kappa(v)}(x_v(s)).
\]

\paragraph{Decision context.}
A correct next step need not be unique, since a branch--merge graph may admit multiple valid topological orderings. Moreover, whether a step is correct does not depend on the name of the node being executed, but only on the skill type invoked and the input it receives. Therefore, for each state, we can define a decision context
\begin{align*}
    Q(s) = \{ (\kappa (v), x_v(s)) : v \in R(s) \},
\end{align*}
which is the set of skill-input pairs available at state $s$. 
We then define the set of valid next invocation-output pairs by

\[
\mathcal I^*(s)
=
\{((\kappa, x), f_{\kappa}(x)): (\kappa, x)\in Q(s)\}.
\]
By construction, $\mathcal I^*(s)$ is determined by $Q(s)$ only. Therefore, we write $\mathcal I^*(c)$ for a decision context $c$. Since $\mathcal S$ is countable, the set of decision contexts $\{Q(s) : s \in \mathcal S\}$ is also countable.

A rollout policy $\hat\pi$ maps a state $s$ to a parsed invocation-output pair, consisting
of the skill invocation it performs, identified by its skill type and input rather than by a node name,
together with the output it produces. When $\hat{\pi}(s) = ((\kappa, x), z) \in \mathcal I^*(s)$, the rollout completes a ready node whose skill type is is $\kappa$ and input is $x$, and records the correct output $z = f_{\kappa}(x)$. If several ready nodes have the same skill type and input, a fixed deterministic tie-breaking rule determines which one is completed, in accordance with our empirical deterministic tie-breaking rule in Appendix~\ref{app:data-generation}. Note that when simultaneously ready nodes induce at least two distinct skill-input pairs, $Q(s)$ is non-singleton. For stochastic policies, all probabilities and expectations below are also over the policy randomness. We assume that, conditioned on the current rollout state $s$, the policy's current randomization is independent of the previous rollout history. Equivalently, any policy-internal memory or persistent random seed that can affect future outputs is included as part of the rollout state. Under this assumption, the stepwise error probability, 
\begin{align*}
    e_{\hat{\pi}}(s) = \Pr_{\hat{\pi}}[\hat{\pi}(s) \notin \mathcal I^*(s)],
\end{align*}
is a well-defined function of state $s$. For a deterministic policy, the error probability is either 0 or 1, so $e_{\hat{\pi}}(s) = \mathbbm{1}\{\hat\pi(s)\notin\mathcal I^*(s)\}$.

\paragraph{Context Invariance.} We assume that the stepwise error probability depends on the state only through the decision context, i.e.,
\begin{align*}
    Q(s) = Q(s') \Rightarrow e_{\hat{\pi}}(s) = e_{\hat{\pi}}(s'),
\end{align*} 
and we write $e_{\hat{\pi}}(c)$ for the stepwise error probability of a decision context $c$. 
This is because our transfer bound compares the policy on rollout distribution with the policy on atomic tasks, and this hypothesis is required to connect the two. By construction, the valid next-step set depends on the state only through its decision context: $Q(s)=Q(s')$ implies $\mathcal I^*(s)= \mathcal I^*(s')$. Thus, context invariance imposes no additional assumption on task correctness; it assumes only that the policy assigns the same total probability to invalid outputs at states with the same decision context. This assumption removes possible additional errors caused by embedding an invocation in a compound prompt, so the resulting transfer guarantee is optimistic. 

\paragraph{Failure and mistake.} Matching the main text, we write $s_t$ for the rollout state at timestep $t$; although $s_t$ is induced by the sampled task $\tau$, we do not attach an additional $\tau$ superscript. The rollout starts from the initial state $s_1$, where $P(s_1)=\emptyset$, and then updates the completed-node set and intermediate outputs according to the parsed pair emitted by $\hat\pi$. When $\hat\pi(s_t)\in\mathcal I^*(s_t)$, the rollout completes one ready node with its correct output; otherwise, the next state may be arbitrary, but the step is counted as a mistake.

Let $F$ denote the final-failure indicator:
\[
F
=
\mathbbm{1}\{\text{the rollout does not produce the correct sink-node output }z_{v_{\mathrm{sink}}}\}.
\]
Equivalently, if the terminal rollout state stores the outputs produced by the serialized rollout, missing or incorrect sink-node output is counted as failure.

A mistake at step $t$ is the event
\[
\hat\pi(s_t)\notin\mathcal I^*(s_t).
\]
Define the indicator of the first mistake made at step $t$ by
\[
M_t=
\mathbbm{1}\!
\left\{
\hat\pi(s_i)\in\mathcal I^*(s_i)\ \text{for all } i<t,
\quad
\hat\pi(s_t)\notin\mathcal I^*(s_t)
\right\}.
\]
If no mistake occurs for $t=1,\ldots,T$, then the rollout completes all nodes in a valid topological order and every completed node has the correct output. In particular, the sink-node output is $z_{v_{\mathrm{sink}}}$, so $F=0$. Hence $F=1$ implies that $M_t=1$ for at least one $t$.

Let $q$ be the distribution over decision contexts induced by atomic task training data. Sample an atomic-task state $s$ and record $Q(s)$. Define 
\[
\epsilon_{\mathrm{atom}}(\hat\pi;q)
=
\mathbb E_{c\sim q}
\left[ e_{\hat{\pi}}(c)\right] = \sum_c q(c) e_{\hat{\pi}}(c).
\]

When $q$ is clear from context, we write $\epsilon_{\mathrm{atom}}(\hat\pi)$. An atomic task instance has a single-node dependency graph, so its decision context is a singleton $\{(\kappa,x)\}$, and $q$ is supported on singleton contexts. A compound rollout context can therefore be covered by $q$ only if it is the same singleton and has positive $q$-mass. This singleton condition holds structurally along Level I and Level II paths, but may fail in Level III when several distinct invocations are simultaneously available. We compare atomic and compound behavior at the level of decision contexts because this captures their possible overlap; it does not assume that every compound decision context has an atomic analogue.

We demonstrate the notation above with a concrete example.
\begin{example}[BST construction]
\label{ex:theory_notation}
Consider the following prompt:
\begin{quote}
\texttt{
You are given an initially empty binary search tree.
Execute the following operations in order: \\
1. insert 5 \\
2. insert 2 \\
3. insert 8 \\
4. delete 2 \\
5. insert 3 \\
Use [] for an empty child and [v,L,R] for a nonempty node with value v,
left subtree L, and right subtree R. Return only the final tree.
}
\end{quote}
This prompt is a task instance $\tau$ with horizon $T=5$. Its dependency graph is a path $v_1\to v_2\to v_3\to v_4\to v_5$, where each node is one BST skill invocation. The local information $\ell_{v_i}$ specifies the instructed operation, such as \texttt{insert 5} for $v_1$ and \texttt{delete 2} for $v_4$. The sink node is $v_{\mathrm{sink}}=v_5$, and the correct sink-node output is
\[
z_{v_{\mathrm{sink}}}
=
[5,[3,[],[]],[8,[],[]]].
\]

Let $B_t$ denote the BST after the first $t-1$ invocations have been correctly executed. Along the correct rollout,
\[
B_1=[],\quad
B_2=[5,[],[]],\quad
B_3=[5,[2,[],[]],[]],\quad
B_4=[5,[2,[],[]],[8,[],[]]],
\]
\[
B_5=[5,[],[8,[],[]]],\quad
B_6=[5,[3,[],[]],[8,[],[]]].
\]
At timestep $t$, the rollout state can be represented as
\[
s_t=\left(\tau,\{v_1,\ldots,v_{t-1}\},\{z_{v_i}=B_{i+1}:i<t\}\right).
\]
Because the graph is a path, the ready-node set is a singleton at every nonterminal correct state: $R(s_t)=\{v_t\}$. For example, at step $4$,
\[
P(s_4)=\{v_1,v_2,v_3\},
\qquad
R(s_4)=\{v_4\}.
\]
so the decision context and the valid next invocation-output set are the singletons
\[
Q(s_4)
=
\left\{
\left(\texttt{delete},(B_4,2)\right)
\right\},
\qquad
\mathcal I^*(s_4)
=
\left\{
\left(\left(\texttt{delete},(B_4,2)\right),B_5\right)
\right\}.
\]
If the model invokes the available skill-input pair but returns an incorrect intermediate BST, or proposes a skill-input pair not in $Q(s_4)$, such as insert 3 before delete 2 is performed, then its invocation-output pair does not belong to $\mathcal I^*(s_4)$. 

The atomic task instance ``given the tree $B_4$, delete $2$'' has the same decision context $Q(s_4)$, even though its dependency graph has one node while that of $\tau$ has five. Decision contexts are exactly the level of description at which an atomic training example and a step of a compound rollout can coincide.
\end{example}

Lemma~\ref{lem:main-coverage} shows that generalizing from atomic to compound tasks depends on both the atomic error on $q$ and the mismatch between $q$ and the decision contexts visited along mistake-free compound rollouts.
\begin{lemmaformal}
\label{thm:sum_of_coverage}
Assume $\hat{\pi}$ is context-invariant. For each step $t$, let $\mu_t$ be the conditional distribution of the decision context $Q(s_t)$ given no mistakes before step $t$:
\[
\mu_t(A)
=
\Pr[Q(s_t)\in A\mid M_1=\cdots=M_{t-1}=0]
\qquad \text{for } \text{a set $A$ of decision contexts}.
\]
Assume that {$q(c)=0\Rightarrow \mu_t(c)=0$} for every $t \in [T]$ and every decision context $c$, i.e., any decision context that can occur at timestep $t$ of a correct compound rollout has nonzero probability under the atomic training-state distribution $q$. Define the cumulative coverage factor
\[
C_T(q)
=
\sup_{c:q(c)>0}\frac{\sum_{t=1}^T} \mu_t(c){q(c)}.
\]
Assume $C_T(q) < \infty$. Then, for $\tau\sim\rho_T$, the probability of final failure satisfies
\[
\Pr[F=1]
\le
 C_T(q) \,\epsilon_{\mathrm{atom}}(\hat\pi;q).
\]
This bound is tight in the worst case, as shown by Proposition~\ref{prop:no-horizon-free-transfer}.
\end{lemmaformal}

\begin{proof}
For each $t$, define the conditional step-error probability
\[
\delta_t(\hat\pi)
=
\Pr[\hat\pi(s_t)\notin\mathcal I^*(s_t)
\mid M_1=\cdots=M_{t-1}=0].
\]
As argued above, if the final output is incorrect, then at least one first mistake must have occurred. Hence
\[
\Pr[F=1]
\le
\sum_{t=1}^{T}\Pr[M_t=1]
=
\sum_{t=1}^{T}\Pr[M_1=\cdots=M_{t-1}=0] \, \delta_t(\hat\pi)
\le
\sum_{t=1}^{T}\delta_t(\hat\pi).
\]

Fix $t$. Conditioned on the rollout state $s_t=s$, the policy makes a mistake with probability $e_{\hat\pi}(s)$, which by context invariance equals $e_{\hat\pi}(Q(s))$. Conditioning instead on the decision context and using that the set of decision contexts is countable,
\[
\delta_t(\hat\pi)
=
\sum_{c}\mu_t(c)\,e_{\hat\pi}(c).
\]
Summing over $t$ and using the coverage assumption, which allows us to restrict the sum to contexts with $q(c)>0$,
\begin{align*}
\sum_{t=1}^{T}\delta_t(\hat\pi)
& =
\sum_{c}\left(\sum_{t=1}^{T}\mu_t(c)\right)e_{\hat\pi}(c)
=
\sum_{c:q(c)>0}
\frac{\sum_{t=1}^{T}\mu_t(c)}{q(c)}\,q(c)\,e_{\hat\pi}(c) \\
& \le
C_T(q)\sum_{c}q(c)\,e_{\hat\pi}(c)
=
C_T(q)\,\epsilon_{\mathrm{atom}}(\hat\pi;q).
\end{align*}
Combining the preceding inequalities proves the claim.

\end{proof}

Lemma~\ref{lem:main-coverage} shows that transferring from atomic supervision to compound tasks depends on two quantities: the atomic error $\epsilon_{\mathrm{atom}}(\hat\pi;q)$ and the coverage factor $C_T(q)$, which measure the mismatch between the atomic training-state distribution and the rollout state distribution. This guarantee is not horizon-free: the horizon dependence is absorbed into the coverage factors $C_T(q)$.

Indeed, suppose $\Pr[M_1 = \dots = M_{t-1} = 0]>0$ for every $t\in[T]$, so that each $\mu_t$ is a probability distribution and $\sum_{c}\sum_{t=1}^T \mu_t(c)=T$. Restricting to contexts with $q(c)>0$, which the coverage assumption permits,
\[
T
=
\sum_{c:q(c)>0}
\frac{\sum_{t=1}^{T}\mu_t(c)}{q(c)}\,q(c)
\le
C_T(q)\sum_{c:q(c)>0}q(c)
=
C_T(q).
\]
So $C_T(q)\ge T$ always, with equality exactly when atomic supervision matches the visitation frequencies of the compound rollout, i.e.\ $\sum_{t}\mu_t=Tq$. In particular, even perfectly matched coverage still pays a factor $T$, recovering the $T\epsilon_{\mathrm{atom}}$-scaling; poorer coverage further amplifies failure risk. Note that the coverage assumption assumes atomic supervision is capable of producing the contexts a compound rollout visits. This holds for path-structured graphs, where $|R(s_t)|=1$ along a correct rollout so every visited context is a singleton, but not in general; we return to this point in the discussion of Level III below.

The following proposition shows that this horizon dependence is unavoidable even in a path-structured, unique-valid-invocation setting.
\begin{prop}[No horizon-free transfer bound]
\label{prop:no-horizon-free-transfer}
For every $T\ge 1$, there exist a compound task distribution $\rho_T$, a countable state space $\mathcal S$, an atomic training distribution $q$ over decision contexts, and a context-invariant policy $\hat\pi$ such that each non-terminal state on the mistake-free rollout has a singleton valid next invocation-output set, yet
\[
\epsilon_{\mathrm{atom}}(\hat\pi;q)=\frac1T,
\qquad
\Pr[F=1]=1.
\]
Consequently, no bound of the form
\[
\Pr[F=1]\le g\!\left(\epsilon_{\mathrm{atom}}(\hat\pi;q)\right)
\]
with $g(\epsilon)\to0$ as $\epsilon\to0$ can hold uniformly over all horizons $T$ and all compound task distributions.
\end{prop}

\begin{proof}
Let $\rho_T$ be the point mass on the following single-skill chain: starting from the empty list, insert the values $1,2,\ldots,T$ in this order. Its dependency graph is a path of length $T$, so the ready-node set is a singleton at every nonterminal correct state, and the correct output of each invocation is unique. Along the correct rollout, the decision context at step $t$ is the singleton
\[
c_t=\left\{\left(\texttt{insert},\left([1,\ldots,t-1],t\right)\right)\right\},
\]
and $c_1,\ldots,c_T$ are pairwise distinct because their inputs are. Each $c_t$ is also the decision context of an atomic task instance, namely the single invocation \texttt{insert} $t$ applied to the list $[1,\ldots,t-1]$, so the uniform distribution $q$ over $\{c_1,\ldots,c_T\}$ is realizable as an atomic training distribution.

Let $\hat\pi$ be the deterministic policy, defined on decision contexts and hence context-invariant, that returns the valid pair on $c_1,\ldots,c_{T-1}$ and a pair outside $\mathcal I^*(c_T)$ on $c_T$. As the rollout model permits, let a mistake send the rollout to a state from which the correct sink-node output is never produced. Then $\hat\pi$ errs on exactly one of the $T$ contexts in the support of $q$, so
\[
\epsilon_{\mathrm{atom}}(\hat\pi;q)=\frac1T.
\]
However, the compound rollout makes no mistake at steps $1,\ldots,T-1$ and therefore visits $c_1,c_2,\ldots,c_T$ in order, necessarily making a mistake at step $T$. Thus, the rollout fails to produce the correct sink-node output, so
\[
\Pr[F=1]=1.
\]
\end{proof}

In this construction, $\mu_t$ is the point mass on $c_t$, so $\sum_{t=1}^{T}\mu_t(c_j)=1$ for any $j \in [T]$, while $q(c_j)=1/T$; hence $C_T(q)=T$ and the bound of Lemma~\ref{lem:main-coverage} evaluates to $T\cdot\frac1T=1=\Pr[F=1]$. Thus Lemma~\ref{lem:main-coverage} is tight in the worst case.

\paragraph{Connection with the compositionality levels.} Lemma~\ref{lem:main-coverage} explains how the difficulty of composing atomic tasks changes across our three levels of compositionality.
\begin{itemize}[leftmargin=*,topsep=2pt, itemsep=2pt, parsep=1pt]
    \item \textbf{Level I (Single-skill chain)}: In this case, $G_{\tau}$ is a path, and all nodes share the same skill type, so the main challenge is the rollout horizon $T=|V_\tau|$ (depth generalization). Even without skill switches or merges, a length-$T$ rollout incurs $T$ context visits. The cumulative factor counts all of these visits, including repeated visits to the same context, and therefore satisfies $C_T(q) \ge T$ under the stated conditions, with equality only under perfectly matched coverage. For the natural uniform mixture over visited contexts, this recovers the familiar $T\epsilon_{\mathrm{atom}}$-type scaling.

    \item \textbf{Level II (Multi-skill chain)}: In this case, the skill invoked at each node $v\in V_\tau$, namely $\kappa(v)$, varies across nodes. The model must therefore not only remain correct over a long horizon, but also execute different skills in sequence and propagate intermediate outputs across skill switches. Even if each skill is well covered in isolation by atomic supervision, the compound rollout distribution near a skill switch may assign mass to contexts in which one skill is applied to an input produced by a different skill, and such cross-skill contexts may receive little mass under the atomic training distribution $q$, increasing $C_T(q)$. This matches our experimental results in Figure~\ref{fig:skill-switch}: as the number of skill switches increases, the task becomes more difficult for the model.

    \item \textbf{Level III (Branch-merge graph)}: 
    In this case, $G_\tau$ is a branch--merge graph, and two effects arise that have no counterpart in Levels I and II. First, whenever simultaneously ready branches induce at least two distinct skill-input pairs, $Q(s_t)$ contains more than one invocation. The model must also choose which of them to execute next, a decision with no atomic analogue. Since an atomic task instance has a single-node dependency graph, its decision context is a singleton, so every such multi-invocation context lies outside the support of any atomic training distribution $q$. The coverage assumption of Lemma~\ref{lem:main-coverage} then fails unless the training distribution is enlarged beyond single-node atomic tasks to include multi-invocation choice contexts. Second, a merge invocation is a single skill applied to a tuple of outputs from different branches, and so is atomically realizable, but the contexts that matter are those in which the tuple was produced by two separately executed branches; if atomic supervision covers merge invocations only on tuples from short or isolated branches, those contexts receive little mass and $C_T(q)$ is large. Both effects are consistent with the collapse of RL-Atomic on Level III compound tasks.
\end{itemize}

\subsection{Omitted Proofs from Section~\ref{sec:co-to-de-easy}}\label{sec:apx_b2}

The previous subsection shows that atomic-to-compound transfer is difficult because errors can accumulate over a multi-step rollout. The reverse direction does not have the same horizon dependence: an error at the first step of a compound rollout is already an error on a single skill invocation.

Let $\rho_T$ denote a distribution over compound task instances with horizon $T$. For $\tau\sim\rho_T$, let $s_1$ be its initial rollout state. Define the compound-task error by
\[
\epsilon_{\mathrm{comp}}(\hat\pi;\rho_T)
=
\mathbb E_{\tau\sim\rho_T,\hat\pi}[F],
\]
and define the first-step  error induced by $\rho_T$ as
\[
\epsilon_{\mathrm{atom}}^{(1)}(\hat\pi;\rho_T)
=
\Pr_{\tau\sim\rho_T,\hat\pi}
\left[
\hat\pi(s_1)\notin\mathcal I^*(s_1)
\right].
\]
To interpret this first-step error as an error with atomic context, we assume that every compound task in the support of $\rho_T$ has a unique source node. Since the completed-node set $P(s_1) = \emptyset$, the ready node set $R(s_1)$ contains exactly one node, and the decision context $Q(s_1)$ is a singleton decision context.

\begin{prop}[First-step compound-to-atomic transfer]
\label{lem:compound-to-atomic-step-one}

Assume that a first-step mistake necessarily prevents successful completion of the compound task, i.e., for every task instance $\tau$ and every realization of the policy randomness,
\[
\hat\pi(s_1)\notin\mathcal I^*(s_1)
\quad\Longrightarrow\quad
F=1.
\]
Then
\[
\epsilon_{\mathrm{atom}}^{(1)}(\hat\pi;\rho_T)
\le
\epsilon_{\mathrm{comp}}(\hat\pi;\rho_T).
\]
\end{prop}

\begin{proof}
Recall that $M_1=\mathbbm{1}
\left\{
\hat\pi(s_1)\notin\mathcal I^*(s_1)
\right\}$ is the indicator of whether the model makes mistake on the first step. By the non-recoverable mistake assumption, $M_1 \le F$ for every task instance and every realization of the policy randomness. Taking the expectation gives $\mathbb E[M_1] \le \mathbb E[F]$, and notice that the left-hand side is $\epsilon_{\mathrm{atom}}^{(1)}(\hat\pi;\rho_T)$ and the right-hand side is $\epsilon_{\mathrm{comp}}(\hat\pi;\rho_T)$.
\end{proof}

Proposition~\ref{lem:compound-to-atomic-step-one} compares two events within the same compound rollout, so no coverage assumption or context invariance assumption is required here. 

\paragraph{Transfer to a separate atomic distribution.} Even when an atomic task and the first step of a compound task require the same skill on the same input, their full rollout states differ because each state records the entire task instance. We therefore compare them through their decision contexts.

Define the first-step decision-context distribution induced by compound tasks as 
\[
    q_1^{\rho_T}(c)
    =
    \Pr_{\tau\sim\rho_T}
    \left[
    Q(s_1)=c
    \right].
\]
Because every task has a unique source node, $q_1^{\rho_T}$ is supported on singleton decision context. 

Assume that $\hat{\pi}$ is context invariant. Then the first-step error can be written as 
\begin{align*}
\epsilon_{\mathrm{atom}}^{(1)}(\hat\pi;\rho_T)
&=
\mathbb E_{\tau\sim\rho_T}
\left[
e_{\hat\pi}(s_1)
\right] \\
&=
\mathbb E_{\tau\sim\rho_T}
\left[
e_{\hat\pi}(Q(s_1))
\right] \\
&=
\sum_c q_1^{\rho_T}(c)e_{\hat\pi}(c).
\end{align*}

Now let $q$ be a target atomic distribution over decision contexts. Assume that $q$ is covered by the compound first-step distribution:
\[
    q_1^{\rho_T}(c)=0 \quad \Longrightarrow \quad q(c)=0.
\]

Define the first-step coverage factor as 
\[
C_1(q,\rho_T)
=
\sup_{c:q_1^{\rho_T}(c)>0}
\frac{q(c)}{q_1^{\rho_T}(c)}.
\]
Then
\[
\epsilon_{\mathrm{atom}}(\hat\pi;q)
\le
C_1(q,\rho_T)\,
\epsilon_{\mathrm{comp}}(\hat\pi;\rho_T).
\]
Indeed, using the definition of atomic error,
\begin{align*}
\epsilon_{\mathrm{atom}}(\hat\pi;q)
&=
\sum_{c}q(c)e_{\hat\pi}(c) \\
&=
\sum_{c:q_1^{\rho_T}(c)>0}
\frac{q(c)}{q_1^{\rho_T}(c)}q_1^{\rho_T}(c)e_{\hat\pi}(c) \\
&\le
C_1(q,\rho_T)
\sum_{c}q_1^{\rho_T}(c)e_{\hat\pi}(c) \\
&=
C_1(q,\rho_T)\,
\epsilon_{\mathrm{atom}}^{(1)}(\hat\pi;\rho_T) \\
&\le
C_1(q,\rho_T)\,
\epsilon_{\mathrm{comp}}(\hat\pi;\rho_T),
\end{align*}
where the final inequality follows from
Proposition~\ref{lem:compound-to-atomic-step-one}.

When $q=q_1^{\rho_T}$, we have $C_1(q,\rho_T)=1$, and therefore
\[
\epsilon_{\mathrm{atom}}(\hat\pi;q_1^{\rho_T})
\le
\epsilon_{\mathrm{comp}}(\hat\pi;\rho_T).
\]
Unlike atomic-to-compound transfer, this bound has no horizon factor: it compares atomic performance with only the first decision context of the compound rollout, rather than accumulating errors over all $T$ steps.

\section{Experiment details}\label{app:experiment-details}

\subsection{Dataset tasks} 

\paragraph{Five data structure domains.}
We chose five common data structure domains of varying difficulty, taken from DSR-Bench \citep{he2026llmsreasonstructurallybenchmarking}: array, binary search tree, bloom filter, hashmap, and heap.

\begin{itemize}[leftmargin=*,topsep=2pt, itemsep=2pt, parsep=1pt]
    \item \textbf{Array.} An array is an ordered list of values indexed from left to right. The model inserts a value at an index or deletes an element, and outputs the final array as a list of integers.

    \item \textbf{Binary Search Tree (BST).} A BST is a binary tree where values in the left subtree are smaller than the node and values in the right subtree are larger. The tree is described through node relationships, and the output is the pre-order traversal of the final tree as a flattened list.

    \item \textbf{Bloom Filter.} We use a counting Bloom filter, which represents a set with an array of counters and several hash functions. Insert increments the counters at the hashed positions, while delete decrements them without going below zero. The output is the final count array.

    \item \textbf{Hashmap.} A hashmap stores key-value pairs in buckets using a hash function. Insert adds or updates a key-value pair, and delete removes a key if it exists. The output is a nested list of buckets, where each bucket contains its key-value pairs.

    \item \textbf{Heap.} We use an array-based min-heap, where each parent is less than or equal to its children. Insert adds a value and restores the heap property by sifting up, while delete removes the root and restores the heap by sifting down. The output is the final heap array.
\end{itemize}

The atomic skills are \textsc{insert} and \textsc{delete} for each domain, with deletion called \textsc{remove} for BSTs and hashmaps in the implementation. The compound task combines these skills: each example gives a sequence of insert and delete operations, and the model must apply them in order and return the final data structure state. 

The problem length is defined as the initial size of the input data structure, such as the number of elements in an array, nodes in a tree, or entries in a hashmap. For compound tasks, the problem length additionally also indicates the number of operations in the chain, since the model must update the data structure repeatedly before producing the final state.

\paragraph{Segment tree construction.}

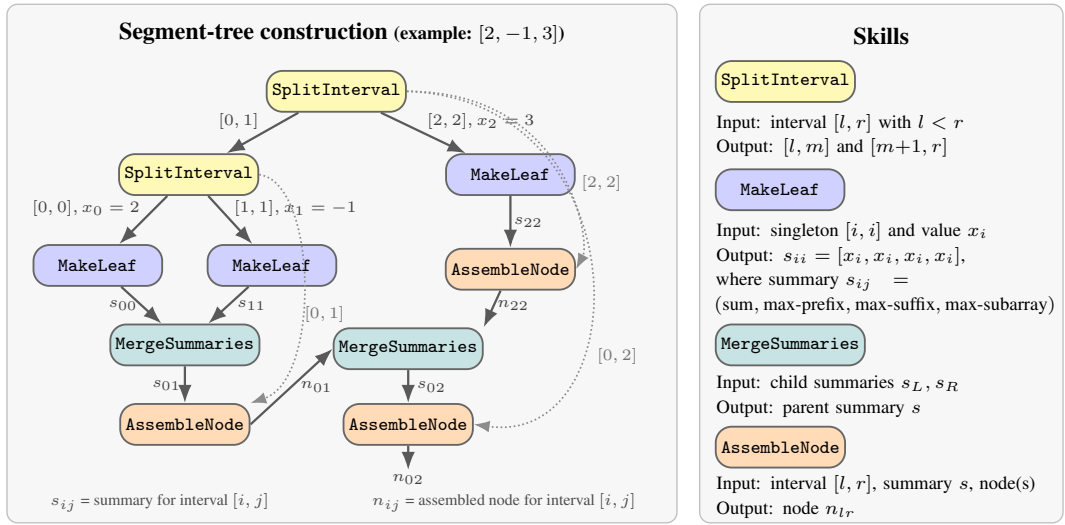
\begin{figure*}[h]
    \centering
    \resizebox{\linewidth}{!}{%
\begin{tikzpicture}[
    x=1cm, y=1cm, >=Latex,
    panel/.style={draw=black!25, fill=black!3, rounded corners=4pt, line width=0.5pt},
    skill/.style={
        draw=black!60, rounded corners=6pt, line width=0.7pt,
        minimum width=1.70cm, minimum height=0.54cm,
        inner sep=2pt, align=center, font=\scriptsize\bfseries, fill=white
    },
    split/.style={skill, fill=yellow!35},
    leaf/.style={skill, fill=blue!18},
    merge/.style={skill, fill=teal!22},
    assemble/.style={skill, fill=orange!28},
    edge/.style={->, line width=0.85pt, draw=black!65},
    aux/.style={->, densely dotted, line width=0.65pt, draw=black!45},
    elab/.style={font=\tiny, inner sep=1pt, text=black!75},
    alab/.style={font=\tiny, inner sep=1pt, text=black!55},
    ltitle/.style={font=\small\bfseries, anchor=west},
    ltext/.style={
        font=\scriptsize,
        align=left,
        text width=3.7cm,
        anchor=north west,
        execute at begin node=\setlength{\baselineskip}{9.3pt}
    }
]

% crop the TikZ bounding box tightly
\path[use as bounding box] (0,0) rectangle (13.95,-6.85);

% ------------------------------------------------------------------
% Panels
% ------------------------------------------------------------------
\draw[panel] (0,0) rectangle (8.85,-6.85);
\draw[panel] (9.15,0) rectangle (13.95,-6.85);

% ------------------------------------------------------------------
% Titles
% ------------------------------------------------------------------
\node[font=\bfseries\small, align=center] at (4.425,-0.38)
    {Segment-tree construction {\scriptsize(example: $[2,-1,3]$)}};
\node[font=\bfseries\small] at (11.55,-0.42) {Skills};

% ------------------------------------------------------------------
% Main DAG (shifted upward)
% ------------------------------------------------------------------
\node[split]    (s0)  at (4.35,-1.15) {\texttt{SplitInterval}};
\node[split]    (s1)  at (2.40,-2.25) {\texttt{SplitInterval}};
\node[leaf]     (l2)  at (6.65,-2.25) {\texttt{MakeLeaf}};

\node[leaf]     (l0)  at (1.20,-3.45) {\texttt{MakeLeaf}};
\node[leaf]     (l1)  at (3.50,-3.45) {\texttt{MakeLeaf}};

\node[merge]    (m01) at (2.35,-4.50) {\texttt{MergeSummaries}};
\node[assemble] (a01) at (2.35,-5.55) {\texttt{AssembleNode}};

\node[assemble] (a2)  at (6.65,-3.50) {\texttt{AssembleNode}};

\node[merge]    (m02) at (5.30,-4.55) {\texttt{MergeSummaries}};
\node[assemble] (a02) at (5.30,-5.55) {\texttt{AssembleNode}};

% ------------------------------------------------------------------
% Main edges
% ------------------------------------------------------------------
\draw[edge] (s0) -- node[elab, above left] {$[0,1]$} (s1);
\draw[edge] (s0) -- node[elab, above right] {$[2,2], x_2=3$} (l2);

\draw[edge] (s1) -- node[elab, above left] {$[0,0], x_0=2$} (l0);
\draw[edge] (s1) -- node[elab, above right] {$[1,1], x_1=-1$} (l1);

\draw[edge] (l0) -- node[elab, left] {$s_{00}$} (m01);
\draw[edge] (l1) -- node[elab, right, xshift=2pt] {$s_{11}$} (m01);
\draw[edge] (m01) -- node[elab, left] {$s_{01}$} (a01);

\draw[edge] (l2) -- node[elab, right] {$s_{22}$} (a2);

\draw[edge] (a01.east) -- node[elab, right, xshift=2pt] {$n_{01}$} (m02.west);

\draw[edge] (a2) -- node[elab, right, pos=0.40] {$n_{22}$} (m02.north east);

\draw[edge] (m02) -- node[elab, right, xshift=2pt] {$s_{02}$} (a02);
\draw[edge] (a02.south) -- ++(0,-0.32) node[elab, below] {$n_{02}$};

% ------------------------------------------------------------------
% Dotted context edges (rerouted to avoid overlap)
% ------------------------------------------------------------------
\draw[aux] (s1.east)
    .. controls (4.10,-2.35) and (4.15,-4.95) ..
    node[alab, right, pos=0.60] {$[0,1]$}
    (a01.north east);

\draw[aux] (s0.east)
    .. controls (7.45,-1.15) and (7.70,-3.15) ..
    node[alab, right, pos=0.55, xshift=5pt] {$[2,2]$}
    (a2.east);

\draw[aux] (s0.east)
    .. controls (8.40,-1.25) and (8.35,-5.85) ..
    node[alab, right, pos=0.68, xshift=5pt] {$[0,2]$}
    (a02.east);

% ------------------------------------------------------------------
% Footer note
% ------------------------------------------------------------------
\node[font=\tiny, text=black!70, anchor=west] at (0.45,-6.58)
    {$s_{ij}$ = summary for interval $[i,j]$};
\node[font=\tiny, text=black!70, anchor=west] at (4.70,-6.58)
    {$n_{ij}$ = assembled node for interval $[i,j]$};

% ------------------------------------------------------------------
% Right panel: skill legend (slightly larger fonts, more even spacing)
% ------------------------------------------------------------------
\node[split, anchor=west] at (9.35,-1.02) {\texttt{SplitInterval}};
\node[ltext]  at (9.25,-1.32)
    {Input: interval $[l,r]$ with $l<r$\\
     Output: $[l,m]$ and $[m{+}1,r]$};

\node[leaf, anchor=west] at (9.35,-2.45) {\texttt{MakeLeaf}};
\node[ltext]  at (9.25,-2.75)
    {Input: singleton $[i,i]$ and value $x_i$\\
     Output: $s_{ii}=[x_i,x_i,x_i,x_i]$, \\
     where summary $s_{ij} = (\text{sum}, \text{max-prefix}, \text{max-suffix}, \text{max-subarray})$};

\node[merge, anchor=west] at (9.35,-4.5) {\texttt{MergeSummaries}};
\node[ltext]  at (9.25,-4.78)
    {Input: child summaries $s_L,s_R$\\
     Output: parent summary $s$};

\node[assemble, anchor=west] at (9.35,-5.85) {\texttt{AssembleNode}};
\node[ltext, text width=6cm]  at (9.25,-6.10)
    {Input: interval $[l,r]$, summary $s$, node(s)\\
     Output: node $n_{lr}$};
\end{tikzpicture}%
}
    \caption{An example of segment-tree construction as branch-merge graph.}
    \label{fig:segment-tree-full}
\end{figure*}

Given an integer array, the model must construct a segment tree over the full index range $[0,n-1]$. A segment tree is a binary tree that recursively partitions an array interval into smaller intervals. In our setting, every node corresponds to an inclusive interval $[l,r]$ and stores a four-field summary
\[
[\texttt{sum}, \texttt{best\_prefix}, \texttt{best\_suffix}, \texttt{best\_subarray}].
\]
Here, \texttt{sum} is the total sum of the interval, \texttt{best\_prefix} is the maximum sum of a non-empty prefix subarray, \texttt{best\_suffix} is the maximum sum of a non-empty suffix subarray, and \texttt{best\_subarray} is the maximum sum of any non-empty contiguous subarray inside the interval. Thus, the constructed tree represents both the recursive interval structure and the dynamic-programming summaries needed for maximum-subarray range queries.

We decompose segment-tree construction into four skills:
\begin{itemize}[leftmargin=*,topsep=2pt, itemsep=2pt, parsep=1pt]
    \item \texttt{SplitInterval}: Given an interval $[l,r]$ with $l<r$, compute $m=\lfloor(l+r)/2\rfloor$ and return the two child intervals $[l,m]$ and $[m+1,r]$.
    \item \texttt{MakeLeaf}: Given a single array value $x$, construct the leaf summary $[x,x,x,x]$, since the sum, best prefix, best suffix, and best subarray are all equal to the only element.
    \item \texttt{MergeSummaries}: Given left summary $[L_{\mathrm{sum}},L_{\mathrm{pref}},L_{\mathrm{suff}},L_{\mathrm{best}}]$ and right summary $[R_{\mathrm{sum}},R_{\mathrm{pref}},R_{\mathrm{suff}},R_{\mathrm{best}}]$, compute the parent summary:
    \[
    \begin{aligned}
    \texttt{sum} &= L_{\mathrm{sum}} + R_{\mathrm{sum}},\\
    \texttt{best\_prefix} &= \max(L_{\mathrm{pref}}, L_{\mathrm{sum}} + R_{\mathrm{pref}}),\\
    \texttt{best\_suffix} &= \max(R_{\mathrm{suff}}, R_{\mathrm{sum}} + L_{\mathrm{suff}}),\\
    \texttt{best\_subarray} &= \max(L_{\mathrm{best}}, R_{\mathrm{best}}, L_{\mathrm{suff}} + R_{\mathrm{pref}}).
    \end{aligned}
    \]
    \item \texttt{AssembleNode}: Package an interval, its summary, and its two children into the node representation $[\texttt{interval}, \texttt{summary}, \texttt{left}, \texttt{right}]$. Leaf children are represented as \texttt{None}.
\end{itemize}

The compound construction task applies these skills recursively. Starting from $[0,n-1]$, if $l=r$, the model applies \texttt{MakeLeaf} and then \texttt{AssembleNode}. If $l<r$, the model applies \texttt{SplitInterval}, independently constructs the left and right subtrees, applies \texttt{MergeSummaries} to combine their summaries, and finally applies \texttt{AssembleNode} to form the parent. Branching occurs at each non-leaf interval when it is split into two child intervals that are processed independently. Merging occurs after both child subtrees are complete, when their summaries are combined into the parent summary.

The final answer is a pre-order traversal of the constructed tree. Each visited node is represented as
\[
[[l,r], [\texttt{sum}, \texttt{best\_prefix}, \texttt{best\_suffix}, \texttt{best\_subarray}]].
\]

\subsection{Example prompt}

Below is an example array-compound task of length 6:

\begin{quote}
\small
\begin{verbatim}
An array supports the following operations (using 0-based indexing):
1. (insert, index=i, value=v) inserts value v at position i,
   shifting elements at i and beyond to the right.
2. (delete, index=i) removes the element currently at position i.
Start with the provided initial array and apply each operation in order.

Initial array: [94, 89, 79, 60, 41, 48]
Operations:
(delete, index=4)
(delete, index=4)
(insert, index=1, value=33)
(insert, index=4, value=26)
(delete, index=3)
(delete, index=0)

Q: What is the final array after applying all operations?
Your answer should be a list of integers, e.g. [1, 2, 3].

Approach the problem methodically. Ensure all conclusions are based on
precise calculations and logical deductions. Feel free to explore various
solution methods and cross-check results for consistency. Maintain dynamic
thinking and always verify each step of your reasoning. The last line of
your response should be of the following format: `Therefore, the final
answer is: $\boxed{ANSWER}$.' (without quotes), where ANSWER is just the
final number or expression that solves the problem. Think carefully and
break down the problem step by step.
\end{verbatim}
\end{quote}

We append the instruction ``Approach the problem methodically...''. This instruction follows \citet{guo2026g1}, which prompts the model to think carefully step by step before giving the final answer.

\subsection{Data generation}\label{app:data-generation}
All examples are synthetically and programmatically generated. This is possible because the behavior of each data structure can be implemented exactly in code, which lets us compute the ground-truth answer without human annotation. Each data structure has a prompt template describing its rules and output format (an example see above). We then populate the template with randomly generated values, using either integers (in the range 0--100) or strings over the English alphabet, depending on the task domain.

For multi-skill chain's compound tasks, each example starts from an initial data structure state and applies a chain of operations. Each operation is sampled as an insert with probability $0.7$ and a delete with probability $0.3$. We ensure that every deletion is valid: for example, array deletion chooses an existing index, BST and hashmap deletion choose existing keys or values, heap deletion is only applied to a non-empty heap, and counting Bloom filter deletion chooses an item that has previously been inserted. The answer is the final serialized state after applying all operations in order.

We also make the outputs deterministic and unique by fixing implementation details that could otherwise be ambiguous. For example, when a traversal has multiple possible next nodes, we visit neighbors in increasing value order; when heap deletion has equal children, we use a fixed tie-breaking rule. These constraints ensure that each prompt has a single well-defined ground-truth output.

\subsection{Output validation}\label{sec:output-validation}
We extract the model's final answer from the final \texttt{\textbackslash boxed\{\}} answer and compare it against the ground truth, following \citet{guo2026g1}. We use instruction-tuned models to reduce errors caused by poor instruction following rather than poor data-structure reasoning. To further reduce formatting sensitivity, our evaluator applies several formatting relaxations:
\begin{itemize}
    \item it extracts the last boxed answer, including boxed answers with nested braces;
    \item it accepts \texttt{\textbackslash boxed} answers written without braces;
    \item it falls back to answers following phrases such as ``the final answer is'';
    \item it strips common wrappers such as code fences and simple LaTeX text wrappers;
    \item it extracts the first balanced list, tuple, or dictionary from otherwise verbose text;
    \item it accepts both Python-style literals and JSON-style literals;
    \item it treats tuples and lists as equivalent for sequence outputs;
    \item for hashmaps, it allows key-value pairs within a bucket to appear in any order.
\end{itemize}

Before comparison, we also normalize the extracted output:
\begin{itemize}
    \item remove extra whitespace, newlines, dollar signs, backslashes, and simple LaTeX spacing commands;
    \item normalize \texttt{\textbackslash tfrac} and \texttt{\textbackslash dfrac} to \texttt{\textbackslash frac};
    \item remove \texttt{\textbackslash left}, \texttt{\textbackslash right}, and degree markers;
    \item normalize decimal forms such as \texttt{.5} to \texttt{0.5};
    \item normalize \texttt{null}/\texttt{None} and Boolean tokens;
    \item parse numeric answers as numbers, with a small tolerance for floating-point answers;
    \item recursively normalize nested sequence structures before equality checking.
\end{itemize}

An example is marked incorrect if the model produces a wrong final state, if no final answer can be extracted, or if the extracted answer cannot be parsed under these formatting relaxations. In preliminary experiments, we also tried Structured Output to allow more flexible answer formats while requiring JSON output, but this lowered performance. We therefore use flattened representations when possible, such as pre-order traversal for BST outputs, since a flat list is easier for models to produce reliably than a deeply nested structure. Future work could use an LLM judge to further reduce the effect of formatting constraints.

\subsection{Reinforcement learning (RL) post-training}

\paragraph{Training framework.}
We perform RL post-training with a rule-based outcome reward. The policy is initialized from an instruction-tuned model and trained with Group Relative Policy Optimization (GRPO) \citep{shao2024deepseekmathpushinglimitsmathematical}. We use \texttt{verl} \citep{sheng2024verl} as the RL training framework, which manages distributed policy updates, reference-policy log probabilities, advantage computation, and checkpointing. For generation during training, we use \texttt{vLLM} \citep{kwon2023vllm} as the rollout engine, which efficiently samples multiple responses from the current policy for each prompt.

For each prompt, the model samples multiple candidate responses. We extract the final boxed answer from each response and score it with the same deterministic verifier used for evaluation. A response receives reward $1$ if the extracted answer matches the ground truth after normalization, and reward $0$ otherwise. The scalar outcome reward is assigned to the final response token. We do not train a learned reward model; instead, the reward is computed directly from programmatic correctness.

\paragraph{Reward function.}
The reward function is a binary correctness signal aligned with the evaluation metric. Given a generated response, we first extract the final answer, prioritizing the final \texttt{\textbackslash boxed\{\}} expression. We then apply the same formatting relaxations and normalization rules used in evaluation (Section \ref{sec:output-validation}), such as stripping simple LaTeX wrappers, accepting JSON- or Python-style list literals, normalizing tuples to lists, and comparing parsed data-structure states rather than raw strings. The normalized prediction is compared against the programmatically generated ground truth. Correct outputs receive reward $1$ and incorrect or unparsable outputs receive reward $0$. This makes the RL signal deterministic and avoids using a learned reward model or an LLM judge during training.

\paragraph{Hyperparameters.}
Training uses a batch size of 128 prompts, with 5 rollouts per prompt for GRPO. The maximum prompt length is 2048 tokens. The maximum response length is 4096, 8192, and 16384 tokens for problem length ranges 5--10, 11--20, and 21--30, respectively. We train for 30 optimization steps with learning rate $5\times 10^{-6}$. We use KL regularization against the reference policy with coefficient $0.001$ and an entropy coefficient of $0.001$. The PPO mini-batch size is 128, and the per-GPU micro-batch size is 4. Experiments were conducted on four NVIDIA H200 141GB GPUs.

\subsection{Supervised finetuning (SFT) post-training}

\paragraph{Data generation.}
We build SFT data from the same synthetic data-structure examples used for RL. For each generated problem, we construct a supervised prompt-response pair where the prompt contains the original data-structure question and the response contains a step-by-step solution trace followed by the final answer. For each data structure and operation type, we manually write SFT reasoning templates for insert and delete behavior. The data generator then fills these templates with the sampled values from each training example, producing a complete chain-of-thought reasoning trace that walks through the operation sequence and ends with the final serialized answer.

\paragraph{Example SFT training data.}
Below is an example array compound SFT instance with both insert and delete operations. After preprocessing, the text before \texttt{A:} is used as the user prompt, and the text after \texttt{A:} is used as the supervised assistant response.

\begin{quote}
\small
\begin{verbatim}
[User prompt]
An array supports the following operations (using 0-based indexing):
1. (insert, index=i, value=v) inserts value v at position i,
   shifting elements at i and beyond to the right.
2. (delete, index=i) removes the element currently at position i.
Start with the provided initial array and apply each operation in order.
Initial array: [52, 74, 58, 7, 63, 79, 31]
Operations:
(insert, index=1, value=52)
(delete, index=3)
...
Q: What is the final array after applying all operations?
Your answer should be a list of integers, e.g. [1, 2, 3].

Approach the problem methodically. Ensure all conclusions are based on
precise calculations and logical deductions. Feel free to explore various
solution methods and cross-check results for consistency. Maintain dynamic
thinking and always verify each step of your reasoning. The last line of
your response should be of the following format: 'Therefore, the final
answer is: $\boxed{{ANSWER}}$.' (without quotes), where ANSWER is just the
final number or expression that solves the problem. Think carefully and
break down the problem step by step.

[Assistant response]
The original array is [52, 74, 58, 7, 63, 79, 31].
The first step is insert 52 at index 1.
Inserting 52 at index 1 shifts the elements from index 1 onwards to the right.
The updated array becomes [52, 52, 74, 58, 7, 63, 79, 31].
The next step is delete at index 3.
Removing the value 58 at index 3 updates the array.
The updated array becomes [52, 52, 74, 7, 63, 79, 31].
...
Therefore, the final answer is:
$\boxed{[52, 52, 73, 77, 7, 72, 63, 79, 68, 31]}$.
\end{verbatim}
\end{quote}

\paragraph{Training.} We train for 300 optimization steps with learning rate $1\times 10^{-5}$, cosine learning-rate scheduling, batch size of 32, weight decay $0.01$, gradient clipping at $1.0$, and 10\% warmup. 

\subsection{The Berkeley Function Calling Leaderboard (BFCL) Benchmark}

We use the Berkeley Function Calling Leaderboard (BFCL) \citep{patil2025bfcl} to evaluate whether a model produces the correct API call given a request. We focus on three Python function-calling categories: \texttt{simple\_python}, \texttt{parallel}, and \texttt{parallel\_multiple}. In BFCL, each function schema specifies a reusable skill primitive. The \texttt{simple\_python} split contains atomic examples, where the user request can be solved by invoking a single function once. The \texttt{parallel} split contains compound examples that require multiple invocations of the same function schema, such as calling a weather API for several cities. The \texttt{parallel\_multiple} split further increases compositionality by providing multiple function schemas and requiring the model to select and invoke several relevant functions. Thus, \texttt{simple\_python} tests isolated skill execution, while \texttt{parallel} and \texttt{parallel\_multiple} test whether the model can coordinate multiple skill invocations within a single user request.

\paragraph{Prompt.} Each sample is converted into a chat-style prompt. The prompt gives the user request and the list of available Python functions, and asks the model to return the final function call or calls. If multiple calls are needed, the model is instructed to put one call per line. The gold answer is stored as a structured JSON payload containing the target calls and whether call order matters.

\paragraph{Data.}
We build separate atomic and compound training and test sets. The atomic training set contains 300 \texttt{simple\_python} examples. The compound training set contains 150 \texttt{parallel} examples and 150 \texttt{parallel\_multiple} examples. The test set contains 100 examples from each of the three categories, for 300 examples in total. 

\paragraph{Verification.}
BFCL answers are verified with a rule-based exact-match checker. The checker parses model outputs from several accepted formats, including Python-style function calls, JSON objects or lists, \texttt{<answer>} tags, and the final \texttt{\textbackslash boxed\{\}} answer format. Parsed calls are normalized before comparison: dictionary keys are sorted, numeric values are canonicalized, and unordered call sets are sorted when the BFCL example is not order-sensitive. A prediction receives reward 1 only if the normalized predicted function name, arguments, and call structure exactly match the normalized gold payload; otherwise it receives reward 0.

\section{Generalization of compositional reasoning} \label{sec:generalization}

In Section~\ref{sec:empirical-compositionality}, we primarily study \emph{length generalization} to evaluate whether models can extend compositional reasoning beyond the training regime. To further probe the generalization induced by composed-task training, we consider two additional transfer settings: robustness under structural distribution shift, and transfer from basic skills to more complex tasks that require unseen skills.

\subsection{Generalization under structural distribution shift} \label{sec:generalization-distribution}

We test whether a model trained on one input distribution can generalize to others while the underlying skill remains unchanged. If the model has learned the skill rather than distribution-specific patterns, its performance should remain robust under such shifts. 

\begin{figure}[h]
\centering
\includegraphics[width=0.8\linewidth]{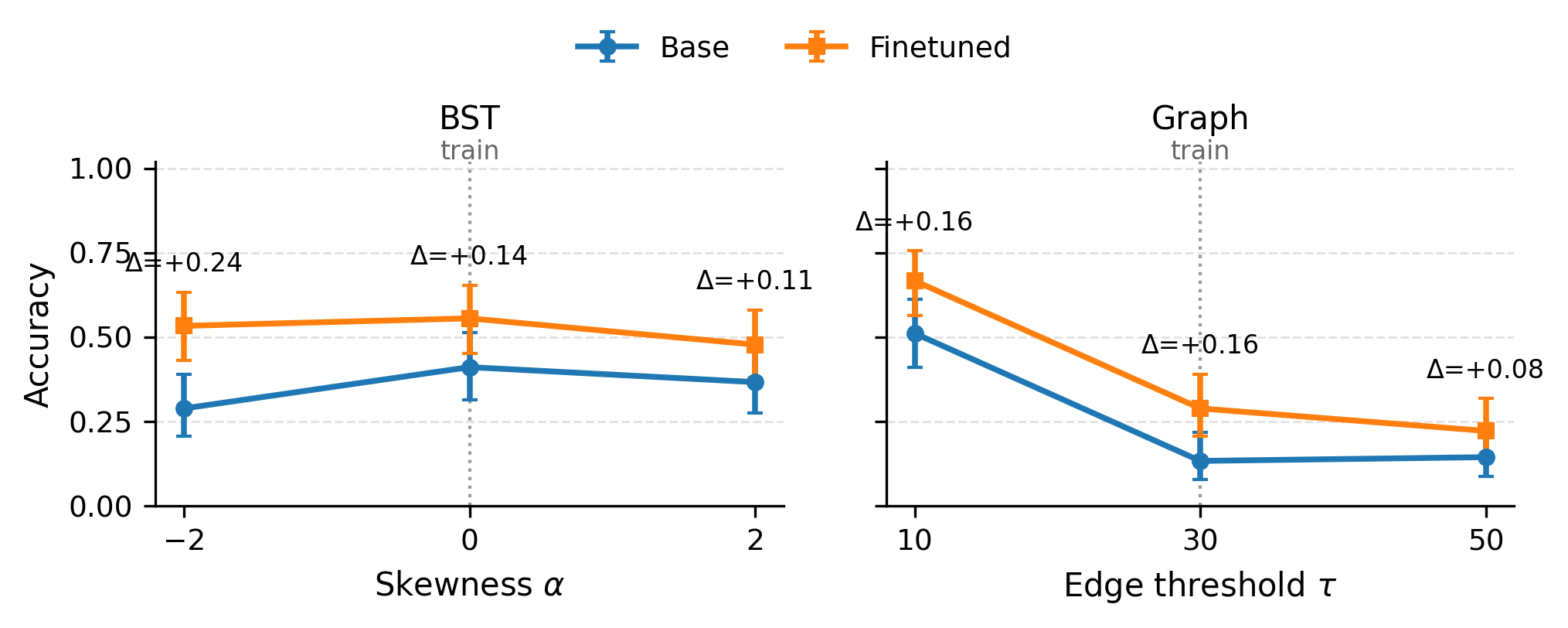}
\caption{Distribution-shift generalization on BST and graph construction. Accuracy across shifts in $\alpha$ and $\tau$. Dotted lines mark the training distribution; error bars show 95\% confidence intervals.}
\label{fig:dist}
\end{figure}

\paragraph{Distribution-shift transfer.} We study two tasks: binary search tree (BST) construction and graph construction. For BSTs, we introduce a parameter $\alpha$ to control tree skewness. When $\alpha > 0$, the generator favors values near the edges of the range, producing more unbalanced insertion orders and hence more skewed trees. When $\alpha < 0$, it favors values near the center, yielding more balanced trees. The case $\alpha = 0$ corresponds to the training distribution. For graph construction, the training distribution uses an edge threshold of 30: an edge is created when the absolute difference between two node values exceeds this threshold. At test time, lowering the threshold to 10 produces denser graphs, while increasing it to 50 produces sparser graphs.

We find that the model finetuned with composed training remains robust under distribution shift, consistently outperforming the base model across all tested distributions in both BST and graph construction. However, the improvement is not uniform across distributions. For BSTs, the gain is largest on more balanced trees and smaller on more skewed trees, suggesting that skewed trees increase execution difficulty, plausibly by requiring a longer implicit stack. For graphs, the gain is also smaller in the sparse regime than in the denser or training-like regimes. Overall, composed training transfers robustly across structural distributions, but its benefit still depends on structural properties that affect execution difficulty.

\subsection{Generalization to unseen skills} \label{sec:generalization-harder}

We study whether training on a set of basic skills enables a model to solve a more complex target task that requires skills unseen during training. We consider two settings.

\begin{figure}[h]
\centering
\begin{subfigure}[b]{0.50\linewidth}
    \centering
    \includegraphics[width=\linewidth]{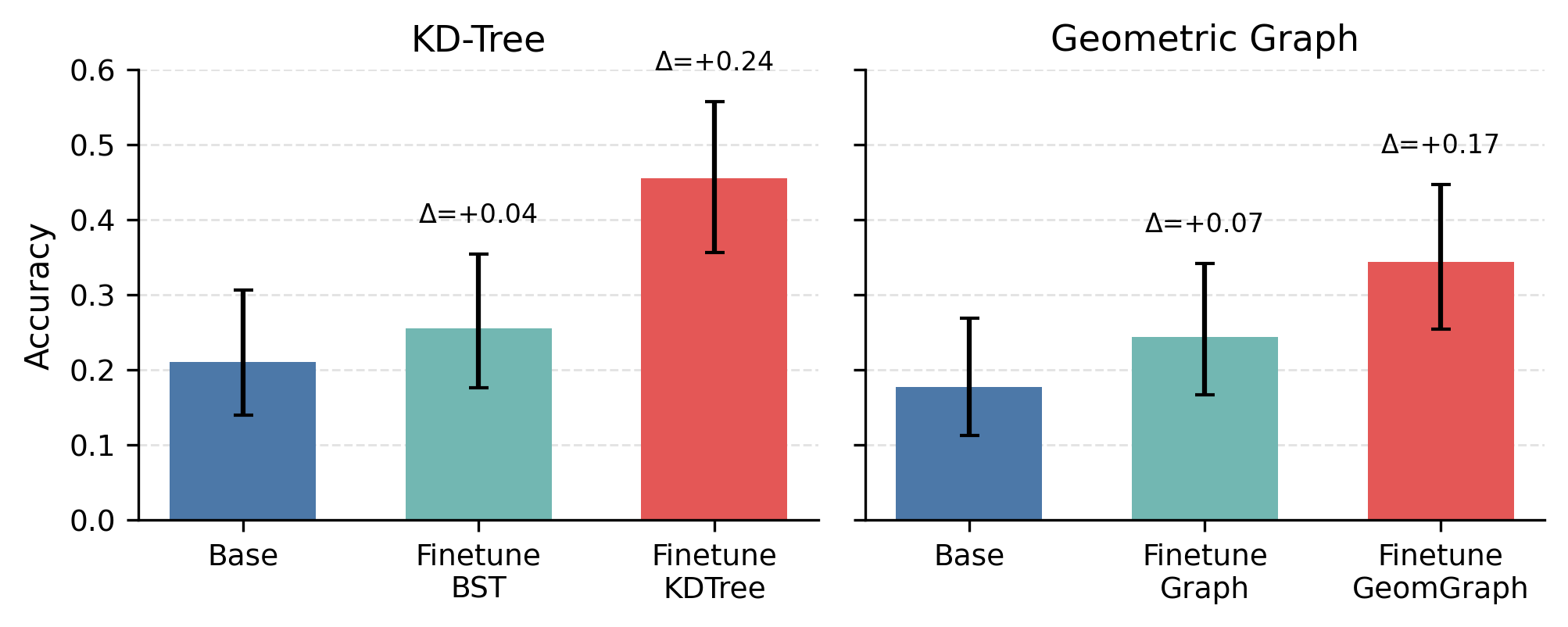}
    \caption{Transfer to structurally richer tasks requiring unseen skills.}
    \label{fig:kdtree-geom}
\end{subfigure}
\hfill
\begin{subfigure}[b]{0.45\linewidth}
    \centering
    \includegraphics[width=\linewidth]{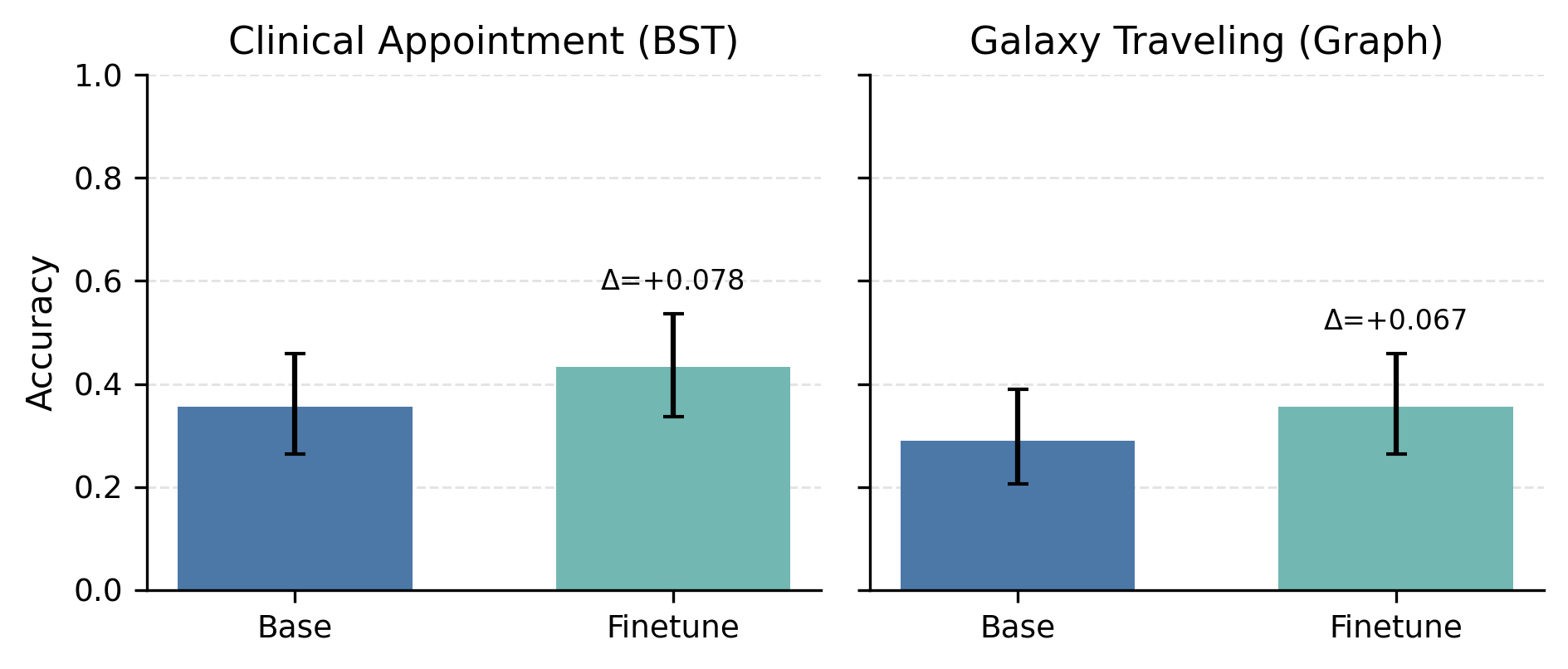}
    \caption{Transfer from formal tasks to natural-language variants.}
    \label{fig:natural}
\end{subfigure}

\caption{Left: transfer to KD-Tree and Geometric Graph construction from basic 1D skills, compared with direct fine-tuning on the target task. Right: transfer from formal tasks to natural-language BST and graph variants. Error bars show 95\% confidence intervals over test examples.}
\label{fig:hard}
\end{figure}

\textbf{Higher-dimensional transfer.} Following the spatial probe in DSR-Bench, we train on BST construction and evaluate on KD-Tree construction. Relative to BSTs, KD-Trees require an additional unseen skill: node-level splitting, where the model must partition a set of points along a specified dimension. For graph construction, we train on the 1D setting and evaluate on Geometric Graphs, a higher-dimensional variant in which edges are determined by Euclidean distance rather than absolute difference. In both cases, we use dimension $K=2$.

\textbf{Natural-language transfer.} Following the realistic probe in DSR-Bench, we embed both tasks into more realistic natural-language settings. For BSTs, the task is framed as patients calling to book clinical appointments, where each node contains a patient name and appointment time. For graphs, the task is framed as spaceship travel across a galaxy through tunnels, where tunnels are created when two planets are sufficiently close in a hyperplane. These variants require the model to map from context-rich language to the underlying formal data-structure problem, and then execute the appropriate reasoning procedure.

We find that training on basic skills does not transfer strongly to harder target tasks that require unseen skills. On KD-Tree and Geometric Graph construction, the gains from transfer are small, while direct fine-tuning on the target task yields much larger improvements. Transfer to natural-language variants is also limited, with only modest gains over the base model. This suggests that compositional transfer remains difficult when generalization requires not only recombining known skills, but also acquiring new operations or reasoning primitives.

\section{Practical implications} \label{sec:practical-guide-data}

Our observed decomposed-to-composed asymmetry has important practical implications for data design in RL post-training, especially as language models are increasingly trained to solve real-world problems that require advanced skills. Such problems usually require skills to be combined, not executed in isolation. In reasoning applications such as \emph{proof writing}, compositional reasoning is required explicitly: success does not come from applying a single theorem or algebraic manipulation, but from choosing relevant lemmas, sequencing them into a valid argument, and maintaining dependencies between intermediate claims. Similar compositional reasoning also underlies real-world agentic applications. In \emph{tool-calling agents}, the challenge is not only to learn individual API calls, such as retrieving weather, searching a database, or updating a calendar entry. The harder problem is to coordinate these calls into a valid workflow, where the output of one step determines what should happen next. Likewise, \emph{coding agents} rarely succeed by editing one line in isolation: they must locate the relevant code, reason about dependencies across functions, apply a patch, and verify that the change is consistent with the intended behavior. Our findings suggest that although individual-skill training can be easier to supervise, it may not be sufficient for tasks that require composition. Training on composed tasks provides a more relevant signal for compositional generalization, while also transferring back to individual skills. Thus, for RL post-training on real-world deployment, data should prioritize verifiable multi-step trajectories rather than only isolated skill drills.

\section{Structured output and intermediate scoring}\label{app:structured-output}

As described in Section~\ref{app:experiment-details}, we extract the final answer using a boxed output format, following \citet{guo2026g1}. To reduce errors caused by formatting issues, we additionally implement normalization and relaxation procedures, and evaluate outputs with instruction-tuned models. These design choices are based on preliminary experiments aimed at identifying the most suitable setup for evaluating compositional reasoning. Below, we discuss the rationale behind these choices.

One alternative is to use JSON-based Structured Output, a feature provided by \texttt{vllm} \citep{kwon2023vllm} that enforces outputs to follow a predefined JSON schema. This approach enables more flexible verification and allows correctness to be checked dynamically from successful executions, instead of relying on manually specified constraints such as always visiting neighbors in ascending order to guarantee a unique output. However, our preliminary experiments with Qwen3-4B-Instruct produces substantially fewer correct outputs when using Structured Output. We hypothesize that the stricter and more complex JSON formatting requirements increase the generation difficulty. Therefore, we do not use Structured Output in our final setup, and instead evaluate serialized outputs (e.g., flattening a tree with preorder traversal) for correctness checking.

Another design choice is that we use only final-output scoring for both RL reward computation and evaluation. Since each step in the data structure computation is deterministic, intermediate states can in principle also be generated and verified. However, doing so requires additional formatting constraints to programmatically extract intermediate outputs. In our preliminary experiments, we found that adding such constraints consistently reduced overall performance. One possible reason is that models may rely on dynamic reasoning processes with self-correction (e.g., ``Aha'' moments), which are disrupted by enforcing structured intermediate outputs. As a result, we evaluate only the final output correctness for our experiments, providing flexibility to the models' reasoning process.

Finally, future work could explore LLM-based judging for answer extraction, removing the need for Structured Output or explicit intermediate-state formatting altogether. This may further reduce formatting-related errors. In addition, tracking intermediate reasoning trajectories and comparing them against ground-truth states---including cases with later error correction---could provide deeper insights into the compositional reasoning abilities of language models.

%\newpage
%\input{checklist.tex}

\end{document}